\documentclass[twoside,11pt]{article}

\PassOptionsToPackage{numbers, compress}{natbib}

\usepackage[preprint]{jmlr2e}

\usepackage{verbatim}
\usepackage{algorithm}
\usepackage{algpseudocode} 
\usepackage{color,graphicx}  
\usepackage{amsfonts}   
\usepackage{amsmath}
\usepackage{amssymb}
\newtheorem{thm}{Theorem}[section]
\newtheorem{lem}{Lemma}[section]
\newtheorem{prop}{Proposition}[section]
\newtheorem{cor}{Corollary}[section]

\newtheorem{defn}{Definition}[section]

\newtheorem{rmk}{Remark}[section]
\newtheorem{ass}{Assumption}[section]

\newcommand{\RR}{\mathbb{R}}      
\newcommand{\CC}{\mathbb{C}}      
\newcommand{\NN}{\mathbb{N}}      
\newcommand{\ZZ}{\mathbb{Z}}      

\newcommand{\vecc}{\boldsymbol}

\usepackage{lastpage}
\jmlrheading{22}{2021}{1-\pageref{LastPage}}{6/20; Revised
1/21}{1/21}{20-620}{Soon Hoe Lim}
\ShortHeadings{title}{Lim}

\ShortHeadings{Understanding Recurrent Neural Networks Using Nonequilibrium Response Theory}{Soon Hoe  Lim}
\firstpageno{1}

\begin{document}

\title{Understanding Recurrent Neural Networks Using Nonequilibrium Response Theory}

\author{\name Soon Hoe Lim 
 \email soon.hoe.lim@su.se \\
       \addr Nordita\\
       KTH Royal Institute of Technology and Stockholm University \\
       Stockholm 106 91, Sweden 
       }

\editor{Sayan Mukherjee}
      
\maketitle

\begin{abstract}
Recurrent neural networks (RNNs) are brain-inspired models widely used in machine learning for analyzing sequential data. The present work is a contribution towards a deeper understanding of how RNNs process input signals using the response theory from nonequilibrium statistical mechanics. For a class of continuous-time  stochastic RNNs (SRNNs) driven by an input signal,  we derive a Volterra type series representation for their output. This representation is interpretable and disentangles the input signal from the SRNN architecture. The kernels of the series are certain recursively defined correlation functions with respect to the unperturbed  dynamics that completely determine the output. Exploiting connections of this representation and its implications to rough paths theory, we identify a universal feature -- the {\it response feature}, which turns out to be the signature of tensor product of the input signal and a natural support basis. In particular, we show that SRNNs, with only the weights in the readout layer optimized and the weights in the hidden layer kept fixed and not optimized, can be viewed as kernel machines operating on a reproducing kernel Hilbert space associated with the response feature.\\

\noindent {\bf Keywords:} Recurrent Neural Networks, Nonequilibrium Response Theory, Volterra Series, Path Signature, Kernel Machines

\end{abstract}

\tableofcontents

\vspace{1cm}

\section{Introduction}
\label{sect_intro}


Sequential data arise in a wide range of settings, from time series analysis to natural language processing. In the absence of a mathematical model, it is important to
extract useful information from the data to learn  the data generating system. 

Recurrent neural networks (RNNs) \citep{Hopfield3088,mcclelland1986parallel,elman1990finding} constitute a class of brain-inspired  models that are specially designed for and widely used for learning sequential data, in fields ranging from the physical sciences to finance.  RNNs are networks of neurons with feedback connections and are arguably  biologically more plausible than other adaptive models.  In particular, RNNs can use their hidden state
(memory) to process variable length sequences of inputs. They are universal approximators of dynamical systems \citep{funahashi1993approximation,schafer2006recurrent,hanson20a} and can themselves be viewed as a class of  open dynamical systems \citep{sherstinsky2020fundamentals}.

Despite their recent innovations  and tremendous empirical success in reservoir computing \citep{herbert2001echo,maass2002real,tanaka2019recent}, deep learning \citep{sutskever2013training,hochreiter1997long,Goodfellow-et-al-2016} and neurobiology \citep{barak2017recurrent}, few studies have focused on the theoretical basis underlying the working mechanism of RNNs. The lack of rigorous analysis limits the usefulness of RNNs in addressing scientific problems and potentially hinders systematic design of the next generation of networks. Therefore, a deep understanding of the mechanism is pivotal to shed light on the properties of large and adaptive architectures, and to revolutionize our understanding of these systems.

In particular, two natural yet fundamental questions that one may ask are:\newpage
\begin{itemize}
\item[(Q1)] {\it How does the output produced by RNNs respond to a driving input signal over time?} 
\item[(Q2)] {\it Is there a universal mechanism underlying their response?}
\end{itemize}

One of the main goals of the present work is to address the above questions, using the nonlinear response theory from nonequilibrium statistical mechanics as a starting point, for a stochastic version of continuous-time RNNs \citep{pineda1987generalization,beer1995dynamics,zhang2014comprehensive}, abbreviated SRNNs, in which the hidden states are injected with a Gaussian white noise. Our approach is cross-disciplinary and adds refreshing perspectives to the existing theory of RNNs.

This paper is organized as follows. In Section \ref{sect_SRNN} we introduce our SRNN model, discuss the related work, and summarize our main contributions.  Section \ref{summary} contains some preliminaries and core ideas of the paper. There  we  derive one of the main results of the paper in an informal manner to aid understanding and to gain intution. We present a mathematical formulation of the main results and other results in Section \ref{sect_main}. We conclude the paper in Section \ref{concl}.  We postpone the technical details, proofs and further remarks to {\bf SM}.

\section{Stochastic Recurrent Neural Networks (SRNNs)} \label{sect_SRNN}
Throughout the paper, we fix a filtered probability space $(\Omega, \mathcal{F}, (\mathcal{F}_t)_{t \geq 0}, \mathbb{P})$, $\mathbb{E}$ denotes expectation with respect to $\mathbb{P}$ and $T>0$.  $C(E,F)$ denotes the Banach space of continuous mappings from $E$ to $F$, where $E$ and $F$ are Banach spaces.  $C_b(\RR^n)$ denotes the space of all bounded continuous functions on $\RR^n$.
$\NN := \{0,1,2,\dots\}$, $\ZZ_{+} := \{1,2,\dots\}$ and $\RR_{+}:= [0,\infty)$.  The superscript $^T$ denotes transposition and $^*$ denotes adjoint.


\subsection{Model}
We consider the following model for our SRNNs. By an {\it activation function}, we mean a real-valued function that is non-constant, Lipschitz continuous and bounded. Examples of activation function include sigmoid functions such as  hyperbolic tangent, commonly used in practice. 

\begin{defn} (Continuous-time SRNNs) \label{CRSRNN}
Let $t \in [0,T]$ and  $\vecc{u} \in  C([0,T],\RR^m)$ be a deterministic input signal. A continuous-time SRNN is described by the following state-space model:
\begin{align}
d\vecc{h}_t &= \vecc{\phi}(\vecc{h}_{t}, t) dt +   \vecc{\sigma} d\vecc{W}_t, \label{e1} \\
\vecc{y}_t &= \vecc{f}(\vecc{h}_t). \label{e2}
\end{align}
In the above, Eq. \eqref{e1} is a stochastic differential equation (SDE) for the hidden states $\vecc{h} = (\vecc{h}_t)_{t \in [0,T]}$, with the drift coefficient  $\vecc{\phi}: \RR^n \times [0,T] \to \RR^n$, noise coefficient $\vecc{\sigma} \in \RR^{n \times r}$, and  $\vecc{W} = (\vecc{W}_t)_{t \geq 0}$ is an $r$-dimensional Wiener process defined on $(\Omega, \mathcal{F},(\mathcal{F}_t)_{t\geq 0}, \mathbb{P})$, whereas Eq. \eqref{e2} defines an observable with $\vecc{f}: \RR^n \to \RR^p$ an activation function. 

We consider an input-affine\footnote{We refer to Theorem C.1 in  \citep{kidger2020neural} for a rigorous justification of considering input-affine continuous-time RNN models. See also Subsection \ref{discretize}, as well as Section IV in \citep{bengio1994learning} and the footnote on the first page of \citep{pascanu2013difficulty} for  discrete-time models.}  version of the SRNNs, in which:
\begin{align}
\vecc{\phi}(\vecc{h}_{t}, t)  = - \vecc{\Gamma} \vecc{h}_t +  \vecc{a}(\vecc{W}\vecc{h}_t + \vecc{b}) + \vecc{C} \vecc{u}_t, \label{e3}
\end{align}
where $\vecc{\Gamma} \in \RR^{n \times n}$ is positive stable, $\vecc{a}: \RR^n \to \RR^n$  is an activation function,  $\vecc{W} \in \RR^{n \times n}$ and $\vecc{b} \in \RR^n$ are constants, and  $\vecc{C} \in \RR^{n \times m}$ is a constant matrix that transforms the input signal.

\end{defn}


From now on, we refer to SRNN as the system defined by \eqref{e1}-\eqref{e3}. The hidden states of a SRNN describe a nonautonomous stochastic dynamical system processing an input signal (c.f.       \citep{ganguli2008memory,dambre2012information,tino2020dynamical}). The constants $\vecc{\Gamma}, \vecc{W}, \vecc{b}, \vecc{C}, \vecc{\sigma}$ and the parameters (if any) in $\vecc{f}$ are the (learnable) parameters or weights defining the (architecture of) SRNN. For $T>0$, associated with the SRNN  is the output functional $F_T:C([0,T],\RR^m) \to \RR^p$ defined as the expectation (ensemble average) of the observable $\vecc{f}$:
\begin{equation}
F_T[\vecc{u}] := \mathbb{E} \vecc{f}(\vecc{h}_T),
\end{equation}
which will be of interest to us.

\subsection{Related Work} \label{discretize}

Our work is in line with the recently promoted approach of ``formulate first, then discretize" in machine learning.  Such approach is popularized in \citep{weinan2017proposal}, inspiring subsequent work \citep{haber2017stable,chen2018neural,rubanova2019latent,benning2019deep,ma2019machine}. Following the approach, here our SRNN model is formulated in the continuous time.

There are several benefits of adopting this approach. At the level of formulation, sampling from these RNNs gives the discrete-time RNNs, including randomized RNNs \citep{herbert2001echo, grigoryeva2018echo,gallicchio2020deep} and fully trained RNNs \citep{bengio2013advances,Goodfellow-et-al-2016}, commonly encountered in applications. More importantly, the continuous-time SDE formulation gives us a guided principle and  flexibility in designing RNN architectures,  in particular those that are capable of adapting to the nature of data (e.g., those irregularly sampled \citep{de2019gru,kidger2020neural,morrill2020neural}) on hand,  going beyond existing architectures. Recent work such as \citep{chang2019antisymmetricrnn,chen2019symplectic,niu2019recurrent,erichson2020lipschitz,rusch2020coupled}  exploits these benefits and designs novel recurrent architectures with desirable stability properties by appropriately discretizing ordinary differential equations. Moreover, in situations where the input data are generated by continuous-time dynamical systems, it is desirable to consider learning models which are also continuous in time. From theoretical analysis point of view, a rich set of tools and techniques  from the continuous-time theory can be borrowed to simplify analysis and to gain useful insights.

The noise injection can  be viewed as a regularization scheme or introducing noise in input data in the context of our SRNNs.  Generally, noise injection can be viewed as a stochastic learning strategy used to improve robustness of the learning model against data perturbations. We refer to, the demonstration of these benefits in, for instance,  \citep{jim1996analysis} for RNNs, \citep{sun2018stochastic} for deep residual networks and \citep{Liu2020} for Neural SDEs. On the other hand, both the continuous-time setting and noise injection are natural assumptions in modelling biological neural networks \citep{cessac2007neuron,touboul2008nonlinear,cessac2019linear}.  Our study here belongs to the paradigm of ``formulate first" (in continuous-time) and  in fact covers both artificial and biological RNNs.

We now discuss in detail an example of how sampling from SRNNs gives rise to a class of discrete-time RNNs. Consider the following SRNN:
\begin{align}
d\vecc{h}_t &= -\gamma \vecc{h}_t dt + \vecc{a}(\vecc{W}\vecc{h}_t  + \vecc{b}) dt + \vecc{u}_t dt + \sigma d\vecc{W}_t, \label{e_eg1} \\
\vecc{y}_t &= \vecc{f}(\vecc{h}_t), \label{e_eg2}
\end{align}
where $\gamma, \sigma > 0$ are constants.
The Euler-Mayurama approximations $(\hat{\vecc{h}})_{t = 0,1,\dots,T}$, with a uniform step size $\Delta t$, to the solution of the SDE \eqref{e_eg1}  are given by \citep{kloeden2013numerical}:
\begin{equation}
    \hat{\vecc{h}}_{t+1} = \alpha \hat{\vecc{h}}_t + \beta (\vecc{a}(\vecc{W} \hat{\vecc{h}}_t + \vecc{b}) +  \vecc{u}_t) + \theta \vecc{\xi}_t,   
\end{equation}
for $t = 0,1,\dots,T-1$, where $\hat{\vecc{h}}_0 = \vecc{h}_0$, $\alpha = 1-\gamma \Delta t$, $\beta = \Delta t$, $\theta = \sqrt{\Delta t} \sigma$ and the $\vecc{\xi}_t$ are i.i.d. standard Gaussian random variables.  In particular, setting $\gamma = \Delta t = 1$ gives:
\begin{equation}
    \hat{\vecc{h}}_{t+1} =    \vecc{a}(\vecc{W} \hat{\vecc{h}}_t + \vecc{b}) +  \vecc{u}_t + \sigma \vecc{\xi}_t.  
\end{equation}
Now, by taking $\hat{\vecc{h}}_t = (\vecc{x}_t, \vecc{z}_t)$, $\vecc{a}(\vecc{W} \hat{\vecc{h}}_t + \vecc{b}) = (\tanh(\vecc{W}'\vecc{x}_t + \vecc{U}' \vecc{z}_t + \vecc{b}'),\vecc{0})$ and $\vecc{u}_t = (\vecc{0}, \vecc{y}_t)$ for some matrices $\vecc{W}'$ and $\vecc{U}'$, some vector $\vecc{b}'$ and input data $(\vecc{y}_t)_{t=0,1,\dots,T-1}$, we have $\vecc{z}_t = \vecc{y}_t$ and 
\begin{equation}
    \vecc{x}_{t+1} = tanh(\vecc{W}'\vecc{x}_t + \vecc{U}' \vecc{y}_t + \vecc{b}') + \sigma \vecc{\xi}_t,
\end{equation}
which is precisely the update equation of a standard discrete-time RNN \citep{bengio1994learning} whose hidden states $\vecc{x}$ are injected by the Gaussian white noise $\sigma \vecc{\xi}$. Note that the above derivation also shows that discrete-time RNNs can be transformed into ones whose update equations are linear in input  by simply introducing additional state variables. In general, different numerical approximations (e.g., those with possibly adaptive step sizes \citep{fang2020adaptive}) to the SRNN hidden states give rise to  RNN  architectures with different properties. Motivated by the above considerations, we are going to introduce a class of noisy RNNs which are obtained by discretizing SDEs and study these benefits, both theoretically and experimentally, for our RNNs in a forthcoming work.

Lastly, we discuss some related work that touches upon  connections between (discrete-time) RNNs and kernel methods. Although connections between non-recurrent neural networks such as two-layer neural networks with random weights and kernel methods are well studied, there are few rigorous studies  connecting  RNNs  and kernel methods. We mention two recent studies here. In the context of reservoir computing, \citep{tino2020dynamical}  analyzes the similarity between reservoir computers and kernel machines and shed some insights on  which network topology gives rise to rich dynamical feature representations.  However, the analysis was  done for only linear non-noisy RNNs.  In the context of deep learning, it is  worth mentioning that \citep{alemohammad2020recurrent} kernelizes RNNs by introducing the Recurrent Neural Tangent Kernel (RNTK) to study the behavior of overparametrized RNNs during their training by gradient descent. The RNTK provides a rigorous connection between the inference performed by infinite width $(n=\infty)$ RNNs and that performed by kernel methods, suggesting that the performance of large RNNs can be replicated by kernel methods for properly chosen kernels. In particular, the derivation of RNTK at initialization is based on the correspondence between randomly initialized infinite width neural networks and Gaussian Processes \citep{neal1996priors}. Understanding the precise learning behavior of finite width RNNs  remains an open problem in the field.

\subsection{Main Contributions}
Our main contributions in this paper are in the following two directions:
\begin{itemize}
\item[(1)] We establish  the relationship between the output functional  of SRNNs and deterministic driving input signal using the response theory. In particular, we derive two series representations for the output functional of SRNNs and their deep version. The first one is a Volterra series representation with the kernels expressed as certain correlation (in time) functions that solely depend on the unperturbed dynamics of SRNNs (see Theorem \ref{thm1}). The second one is in terms of a series of iterated integrals of a transformed input signal (see Theorem \ref{thm2}).  These representations are interpretable and allow us to gain insights into the working mechanism of SRNNs.
\item[(2)] Building on our understanding in (1), we identify a universal feature, called the {\it response feature}, potentially useful for learning temporal series. This feature turns out to be the signature of tensor product of the input signal and a  vector whose components are orthogonal polynomials (see Theorem \ref{sigrep}). We then show that SRNNs, with only the weights in the readout layer optimized and the weights in the hidden layer kept fixed and not optimized, are essentially kernel machines operating in a reproducing kernel Hilbert space associated with the response feature (see Theorem \ref{rep}). This result characterizes precisely the idea that SRNNs with hidden-layer weights that are fixed and not optimized can be viewed as kernel methods. 
\end{itemize}

In short, we focus on studying representations for output functionals of SRNNs and relating SRNNs to certain kernel machines in this work.
To achieve (1) we develop and make rigorous nonlinear response theory for the SRNNs \eqref{e1}-\eqref{e3},  driven by a small deterministic input signal (see {\bf SM}). This  makes rigorous the results of existing works on response theory in the physics literature and also extend the recent rigorous work of \citep{chen2020mathematical} beyond the linear response regime, which may be of independent interest.

For simplicity we have chosen to work with the SRNNs under a set of rather restricted but reasonable assumptions, i.e., Assumption \ref{imp_ass}, in this paper. Also, one could possibly work in the more general setting where the driving input  signal is a rough path \citep{lyons2002system}. Relaxing the assumptions here and working in the more general setting come at a higher cost of technicality and risk burying intuitions, which we avoid in the present work.

\section{Nonequilibrium Response Theory of SRNNs}
\label{summary}

\subsection{Preliminaries and Notation}



In this subsection we briefly recall preliminaries on Markov processes \citep{karatzas1998brownian,pavliotis2014stochastic} and introduce some of our notations. 

Let $t \in [0,T]$, $\gamma(t) := |\vecc{C} \vecc{u}_t| > 0$  and $\vecc{U}_t := \vecc{C}\vecc{u}_t/|\vecc{C} \vecc{u}_t|$ be a normalized  input signal (i.e., $|\vecc{U}_t| = 1$). In the SRNN \eqref{e1}-\eqref{e3},  we consider the  signal $\vecc{C} \vecc{u}  = (\vecc{C} \vecc{u}_t)_{t \in [0,T]}$ to be a perturbation with small amplitude $\gamma(t)$ driving the SDE:
\begin{align}
d\vecc{h}_t &=  \vecc{\phi}(\vecc{h}_t,t) dt  +   \vecc{\sigma} d\vecc{W}_t.
\end{align}
The unperturbed SDE is the system with $\vecc{C} \vecc{u}$ set to zero: 
\begin{align}
d\overline{\vecc{h}}_t &=  \overline{\vecc{\phi}}(\overline{\vecc{h}}_t) dt +   \vecc{\sigma} d\vecc{W}_t.
\end{align}
In the above, $\overline{\vecc{\phi}}(\vecc{h}_t) = -\vecc{\Gamma} \vecc{h}_t +  \vecc{a}(\vecc{W}\vecc{h}_t  + \vecc{b})$ and $\vecc{\phi}(\vecc{h}_t,t) = \overline{\vecc{\phi}}(\vecc{h}_t) + \gamma(t)  \vecc{U}_t$.  The process $\vecc{h}$ is a perturbation of the time-homogeneous Markov process $\overline{\vecc{h}}$, which is not necessarily stationary.



The diffusion process $\vecc{h}$ and $\overline{\vecc{h}}$ are associated with a family of infinitesimal generators $(\mathcal{L}_t)_{t \geq 0}$ and $\mathcal{L}^0$ respectively, which are second-order elliptic operators defined by:
\begin{align}
\mathcal{L}_t f &=  \phi^i(\vecc{h},t) \frac{\partial}{\partial h^i} f + \frac{1}{2} \Sigma^{ij}  \frac{\partial^2}{\partial h^i \partial h^j} f, \label{7} \\
\mathcal{L}^0 f &= \bar{\phi}^i(\overline{\vecc{h}}) \frac{\partial}{\partial \overline{h}^i} f + \frac{1}{2} \Sigma^{ij} \frac{\partial^2}{\partial \overline{h}^i \partial \overline{h}^j} f, \label{8}
\end{align}
for any observable $f \in C_b(\RR^n)$, where  $\vecc{\Sigma} := \vecc{\sigma}\vecc{\sigma}^T > 0$.
We define the transition operator $(P_{s,t})_{s \in [0,t]}$ associated with $\vecc{h}$ as:
\begin{equation}
P_{s,t} f(\vecc{h}) = \mathbb{E}[f(\vecc{h}_t) | \vecc{h}_s = \vecc{h}] \label{srnn2},
\end{equation}
for  $f \in C_b(\RR^n)$, and similarly for the transition operator $(P^0_{s,t})_{s \in [0,t]}$ (which is a Markov semigroup) associated with $\overline{\vecc{h}}$. 

Moreover, one can define the  $L^2$-adjoint of the above generators and transition operators on the space of probability measures. We denote the adjoint generator associated to $\vecc{h}$ and $\overline{\vecc{h}}$ by $\mathcal{A}_{t}$ and $\mathcal{A}^0$ respectively, and the  adjoint transition operator  associated to $\vecc{h}$ and $\overline{\vecc{h}}$ by $(P^*_{s,t})_{s \in [0,t]}$ and $((P^0_{s,t})^*)_{s \in [0,t]}$ respectively. We assume that the initial measure and the law of the processes have a density with respect to Lebesgue measure. Denoting the initial density as $\rho(\vecc{h},t=0) = \rho_{init}(\vecc{h})$,   $\rho(\vecc{h},t) = P_{0,t}^* \rho_{init}(\vecc{h})$ satisfies a forward Kolmogorov equation (FKE) associated with $\mathcal{A}_t$.  

We take the natural assumption that both perturbed and unperturbed process have the same initial distribution $\rho_{init}$, which is generally not the invariant distribution $\rho_\infty$ of the unperturbed dynamics. 


\subsection{Key Ideas and Formal Derivations}
\label{formal}

This subsection serves to provide a {\it formal} derivation of one of the core ideas of the paper in an explicit manner to aid understanding.   

First, we are going to derive a representation for the output functional  of  SRNN in terms of driving input signal.  Our approach is based on the response theory originated from nonequilibrium statistical mechanics. For a brief overview of this theory, we refer to Chapter 9 in \citep{pavliotis2014stochastic} (see also 
\citep{kubo1957statistical,peterson1967formal,hanggi1978stochastic,baiesi2013update}). In the following, we assume that any infinite series is well-defined and any interchange between  summations and integrals is justified.  
 
 




Fix a $T>0$. Let $\epsilon := \sup_{t \in [0,T]} |\gamma(t)| > 0$ be sufficiently small and\
\begin{equation} \tilde{\vecc{U}}_t := \frac{\vecc{C} \vecc{u}_t}{\sup_{t \in [0,T]}|\vecc{C} \vecc{u}_t|}.
\end{equation}

To begin with, note that   the FKE for the probability density $\rho(\vecc{h},t)$ is:
\begin{equation}
\frac{\partial \rho}{\partial t} = \mathcal{A}_t^\epsilon \rho, \ \ \ \rho(\vecc{h},0)=\rho_{init}(\vecc{h}),
\end{equation}
where $\mathcal{A}_t^\epsilon = \mathcal{A}^0 + \epsilon \mathcal{A}_t^1$, with:
\begin{align}
\mathcal{A}^0 \cdot &= - \frac{\partial}{\partial h^i} \left( \bar{\phi}^i(\vecc{h}) \cdot \right) +  \frac{1}{2} \Sigma^{ij} \frac{\partial^2}{\partial h^i \partial h^j}  \cdot , \ \text{ and} \ \ \  \mathcal{A}_t^1 \cdot = - \tilde{U}_t^i \frac{\partial}{\partial h^i}\cdot.
\end{align}


The key idea is that since $\epsilon > 0$ is small we seek a perturbative expansion for $\rho$ of the form 
\begin{equation}
\rho = \rho_0 + \epsilon \rho_1 + \epsilon^2 \rho_2 + \dots.
\end{equation}
Plugging this into the FKE and matching orders in $\epsilon$, we obtain the following hierachy of equations:
\begin{align}
\frac{\partial \rho_0}{\partial t} &= \mathcal{A}^0 \rho_0, \ \ \rho_0(\vecc{h}, t=0) = \rho_{init}(\vecc{h}); \\
\frac{\partial \rho_n}{\partial t} &= \mathcal{A}^0 \rho_n + \mathcal{A}^1_t \rho_{n-1}, \ \ n=1,2,\dots.
\end{align}

The formal solution to the $\rho_n$ can be obtained iteratively as follows. Formally, we write $\rho_0(\vecc{h},t) = e^{\mathcal{A}^0 t} \rho_{init}(\vecc{h})$. In the special case when the invariant distribution is stationary, $\rho_0(\vecc{h},t) = \rho_{init}(\vecc{h}) = \rho_{\infty}(\vecc{h})$ is independent of time. 

Noting that $\rho_n(\vecc{h},0) = 0$ for $n \geq 2$, the solutions $\rho_n$ are related recursively via:
\begin{equation}
\rho_n(\vecc{h},t) = \int_0^t e^{\mathcal{A}^0(t-s)} \mathcal{A}^1_s \rho_{n-1}(\vecc{h},s) ds, \label{recur}
\end{equation}
for $n \geq 2$.  Therefore, provided that the infinite series below converges absolutely, we have:
\begin{equation}
\rho(\vecc{h},t) = \rho_{0}(\vecc{h},t) + \sum_{n=1}^{\infty} \epsilon^n \rho_n(\vecc{h},t). \label{15}
\end{equation}

Next we consider a scalar-valued observable\footnote{The extension to the vector-valued case is straightforward.}, $\mathcal{F}(t) := f(\vecc{h}_t)$,  of the hidden dynamics of the SRNN and study the deviation of average of this observable caused by the perturbation of input signal:
\begin{equation}
 \mathbb{E} \mathcal{F}(t) - \mathbb{E}^0 \mathcal{F}(t) := \int f(\vecc{h}) \rho(\vecc{h},t) d\vecc{h} - \int f(\vecc{h}) \rho_{0}(\vecc{h},t) d\vecc{h}.
\end{equation}
Using \eqref{15}, the average of the observable with respect to the perturbed dynamics can be written as:
\begin{align} \label{avg}
\mathbb{E} \mathcal{F}(t) &= \mathbb{E}^0\mathcal{F}(t) + \sum_{n=1}^{\infty} \epsilon^n \int f(\vecc{h}) \rho_n(\vecc{h},t) d\vecc{h}.
\end{align}
Without loss of generality, we take $\mathbb{E}^0 \mathcal{F}(t) = e^{\mathcal{A}^0 t} \int \rho_{init}(\vecc{h}) f(\vecc{h}) d\vecc{h} = 0$ in the following, i.e., $f(\vecc{h})$ is taken to be mean-zero (with respect to $\rho_{init}$). 

We have:
\begin{align}
\int f(\vecc{h}) \rho_1(\vecc{h},t) d\vecc{h} &= \int d\vecc{h} \ f(\vecc{h})  \int_0^t ds \  e^{\mathcal{A}^0(t-s)} \mathcal{A}^1_s  e^{\mathcal{A}^0 s} \rho_{init}(\vecc{h}) \\
&= -\int_0^t ds \   \int d\vecc{h} \ f(\vecc{h}) e^{\mathcal{A}^0(t-s)} \tilde{U}_s^j  \frac{\partial}{\partial h^j}\left(    e^{\mathcal{A}^0 s} \rho_{init}(\vecc{h}) \right) \\
&=: \int_0^t  \mathcal{K}^{k}_{\mathcal{A}^0,\mathcal{F}}(t,s) \tilde{U}_s^k  ds,
\end{align}
where the
\begin{align}
\mathcal{K}^{k}_{\mathcal{A}^0, \mathcal{F}}(t,s) &= - \int d\vecc{h} \ f(\vecc{h}) e^{\mathcal{A}^0(t-s)}  \frac{\partial}{\partial h^k}\left(  e^{\mathcal{A}^0 s} \rho_{init}(\vecc{h})\right) \\
&= -\left\langle e^{\mathcal{L}^0(t-s)} f(\vecc{h}) \frac{\partial}{\partial h^k}\left(  e^{\mathcal{A}^0 s} \rho_{init}(\vecc{h}) \right) \rho_{init}^{-1}(\vecc{h})  \right\rangle_{\rho_{init}}, \label{AFDT}
\end{align}
are the {\it first-order response kernels},  which are averages,  with respect to $\rho_{init}$, of a functional of only the unperturbed dynamics. Note that in order to obtain the last line above we have integrated by parts and  assumed that $\rho_{init} > 0$. 

Formula \eqref{AFDT} expresses the  {\it nonequilibrium fluctuation-dissipation relation} of Agarwal type \citep{agarwal1972fluctuation}. In the case of stationary invariant distribution, we recover the well-known equilibrium fluctuation-dissipation relation in statistical mechanics, with the (vector-valued) response kernel:
\begin{align}
\vecc{\mathcal{K}}_{\mathcal{A}^0, \mathcal{F}}(t,s) 
&= \vecc{\mathcal{K}}_{\mathcal{A}^0, \mathcal{F}}(t-s)  = \left \langle f(\vecc{h}_{t-s})  \vecc{\nabla}_{\vecc{h}} L(\vecc{h}_{t-s}) \right \rangle_{\rho^\infty},
\end{align}
where $L(\vecc{h}) = -\log \rho_\infty(\vecc{h})$. In the special case of linear SRNN (i.e., $\vecc{\phi}(\vecc{h},t)$  linear in $\vecc{h}$) and $f(\vecc{h}) = \vecc{h}$, this essentially reduces to the covariance function (with respect to $\rho^\infty$) of $\vecc{h}_{t-s}$.

So far we have looked at the linear response regime, where the response depends linearly on the input. 
We now go beyond this regime by extending the above derivations to the case of $n \geq 2$. Denoting $s_0 := t$ and applying \eqref{recur}, we derive:
\begin{align}
\int f(\vecc{h}) \rho_n(\vecc{h},t) d\vecc{h} &= \int_0^{t} ds_1 \tilde{U}_{s_1}^{k_1} \int_0^{s_1} ds_2 \tilde{U}_{s_2}^{k_2} \dots \int_0^{s_{n-1}} ds_{n} \tilde{U}_{s_{n}}^{k_{n}} \mathcal{K}^{ \vecc{k}^{(n)}}(s_0,s_1,\dots,s_{n}),
\end{align}
where  $\vecc{k}^{(n)} := (k_1,\dots,k_{n})$, and the $\mathcal{K}^{\vecc{k}^{(n)}}$ are the {\it $n$th order response kernels}:
\begin{align}
\mathcal{K}^{\vecc{k}^{(n)}}(s_0,s_1,\dots,s_{n}) = (-1)^n \left\langle f(\vecc{h}) \rho^{-1}_{init}(\vecc{h}) R^{\vecc{k}^{(n)}}(s_0,s_1,\dots,s_{n})  e^{\mathcal{A}^0_{s_n}} \rho_{init}(\vecc{h})  \right\rangle_{\rho_{init}}, \label{24}
\end{align}
with 
\begin{align}
R^{k_1}(s_0,s_1) \cdot &= e^{\mathcal{A}^0(s_0 - s_{1})}  \frac{\partial}{\partial h^{k_{1}}} \cdot, \\ 
R^{\vecc{k}^{(n)}}(s_0,s_1, \dots, s_{n}) \cdot &= R^{\vecc{k}^{(n-1)}}(s_0,s_1, \dots, s_{n-1}) \cdot \left( e^{\mathcal{A}^0(s_{n-1} - s_{n})} \frac{\partial}{\partial h^{k_{n}}} \cdot  \right), \label{26}
\end{align}
for $n=2,3,\dots$. Note that these higher order response kernels, similar to the first order ones, are averages, with respect to $\rho_{init}$, of some functional of only the unperturbed dynamics. 

Collecting the above results, \eqref{avg} becomes a series of generalized convolution integrals, given by:
\begin{align} \label{volt}
\mathbb{E} f(\vecc{h}_t) &=  \sum_{n=1}^{\infty} \epsilon^n   \int_0^t ds_1 \tilde{U}_{s_1}^{k_1} \int_0^{s_1} ds_2 \tilde{U}_{s_2}^{k_2} \dots \int_0^{s_{n-1}} ds_{n} \tilde{U}_{s_{n}}^{k_{n}} \mathcal{K}^{ \vecc{k}^{(n)}}(s_0, s_1,\dots,s_{n}),
\end{align}
with the time-dependent kernels $\mathcal{K}^{\vecc{k}^{(n)}}$ defined recursively via \eqref{24}-\eqref{26}. {\it More importantly, these kernels are completely determined in an explicit manner by the unperturbed dynamics of the SRNN.} Therefore, the output functional of SRNN can be written (in fact, uniquely) as a series of the above form. This statement is formulated precise in Theorem \ref{thm1}, thereby addressing (Q1). 


We now address (Q2). By means of expansion techniques, one can derive (see Section \ref{sect_b4}):
\begin{align}
&\mathbb{E} f(\vecc{h}_t) = \sum_{n=1}^{\infty}  \epsilon^n Q^{\vecc{k}^{(n)}}_{\vecc{p}^{(n)}} \left( s_0^{p_0}  \int_0^{s_0} ds_1    s_1^{p_1}  \tilde{U}_{s_1}^{k_1}    \cdots \int_0^{s_{n-1}} ds_{n}  s_{n}^{p_{n}}  \tilde{U}_{s_{n}}^{k_{n}} \right), \label{ultimate_main}
\end{align}
where the $Q^{\vecc{k}^{(n)}}_{\vecc{p}^{(n)}}$ are constants independent of time and the signal $\tilde{\vecc{U}}$. This  expression disentangles the driving input signal from the SRNN architecture in a systematic manner. Roughly speaking, it tells us that the response of SRNN to the input signal can be obtained by adding up  products of two components, one of which describes the unperturbed part of SRNN (the $Q^{\vecc{k}^{(n)}}_{\vecc{p}^{(n)}}$ terms in \eqref{ultimate_main}) and the other one is an iterated integral of a time-transformed input signal (the terms in parenthesis in \eqref{ultimate_main}). This statement is made precise in Theorem \ref{thm2}, which is the starting point to addressing (Q2). See further  results and discussions in Section \ref{sect_main}.

\section{Main Results}\label{sect_main}

\subsection{Assumptions}

For simplicity and intuitive appeal we work with the following rather restrictive assumptions on the SRNNs \eqref{e1}-\eqref{e3}. These assumptions can be  either relaxed at an increased cost of technicality (which we do not pursue here) or justified by the approximation result in Section \ref{approx}.

Recall that we are working with a deterministic input signal $\vecc{u} \in C([0,T],\RR^m)$. 

\begin{ass} \label{imp_ass} Fix a $T>0$ and let $U$ be an open set in $\RR^n$.  \\
(a) $\gamma(t) :=|\vecc{C} \vecc{u}_t| > 0$ is sufficiently small for all $t \in [0,T]$. \\
(b) $\vecc{h}_t, \overline{\vecc{h}}_t \in U$ for all $t \in [0,T]$, and, with probability one, there exists a compact set $K \subset U$ such that, for all $\gamma(t)$, $\vecc{h}_t, \overline{\vecc{h}}_t  \in K$ for all $t \in [0,T]$. \\
(c) The coefficients $\vecc{a} : \RR^n \to \RR^n$ and $\vecc{f}: \RR^n \to \RR^p$ are analytic functions. \\
(d) $\vecc{\Sigma} := \vecc{\sigma}\vecc{\sigma}^T \in \RR^{n \times n}$ is positive definite and  $\vecc{\Gamma} \in \RR^{n \times n}$ is positive stable (i.e., the real part of all eigenvalues of $\vecc{\Gamma}$ is positive).\\
(e) The initial state $\vecc{h}_0 = \overline{\vecc{h}}_0$ is a random variable distributed according to the probability density $\rho_{init}$.  
\end{ass}

Assumption \ref{imp_ass} (a) implies that we work with  input signals with  sufficiently small amplitude. This is important to ensure that certain infinite series are absolutely convergent with a radius of convergence that is sufficiently large (see the statements after Definition \ref{def_vol}).  (b) and (c) ensure some desirable regularity and boundedness properties. In particular, they imply that $\vecc{a}, \vecc{f}$ and all their partial derivatives are bounded\footnote{Boundedness of these coefficients will be important when deriving the estimates here. }  and Lipschitz continuous in $\vecc{h}_t$ and $\overline{\vecc{h}}_t$ for all $t \in [0,T]$.   (d) implies that the system is damped and driven by a nondegenerate noise, ensuring that the unperturbed system could be  exponentially  stable. (e) is a natural assumption for our analysis since $\vecc{h}$ is a perturbation of $\overline{\vecc{h}}$. 

Assumption \ref{imp_ass} is implicitly assumed throughout the  paper unless stated otherwise. \\

\noindent {\bf Further Notation.} We now provide a list of the spaces and their notation that we will need from now on in the main paper and the {\bf SM}:
\begin{itemize}
    \item $L(E_1, E_2)$: the Banach space of bounded linear operators from $E_1$ to $E_2$ (with $\|\cdot\|$ denoting norms on appropriate spaces) 
    \item $C_c^n(0,t)$, $n \in \ZZ_+ \cup \{\infty\}$, $t>0$:  the space of real-valued  functions of class $C^n(0,t)$ with compact support 
    
     \item $C_b^n(0,t)$, $n \in \ZZ_+ \cup \{\infty\}$, $t>0$:  the space of bounded real-valued functions of class $C^n(0,t)$ 
    
    \item $B(\RR_{+}^n)$: the space of bounded absolutely continuous measures on $\RR_+^n$, with $|\mu| = \int d |\mu| = \int |\rho(x)| dx$, where $\rho$ denotes the density of the measure $\mu$ 
    \item $L^p(\rho)$, $p>1$: the $\rho$-weighted $L^p$ space, i.e., the space of functions $f$ such that $\|f\|_{L^p(\rho)} := \int |f(x)|^p \rho(x) dx < \infty$, where $\rho$ is a weighting function
\end{itemize}

\subsection{Representations for Output Functionals of SRNNs}
Without loss of generality, we are going to take $p=1$ and assume that $\int f(\vecc{h}) \rho_{init}(\vecc{h}) d\vecc{h} = 0$ in the following.

First, we define  response functions of an observable, extending the one formulated in \citep{chen2020mathematical} for the linear response regime.   Recall that for $t > 0$, $\gamma := (\gamma(s) := |\vecc{C} \vecc{u}_s|)_{s \in [0,t]}$. 

\begin{defn} (Response functions) \label{def_res}
Let $f: \RR^n \to \RR$ be a bounded observable. For $t \in [0,T]$, let $F_t$ be the functional on $C([0,t],\RR)$ defined as $F_t[\gamma] = \mathbb{E} f(\vecc{h}_t)$ and $D^n F_t[\gamma] := \delta^n F_t/\delta \gamma(s_1) \cdots \delta \gamma(s_n)$ denote the $n$th order functional derivative of $F_t$ with respect to $\gamma$ (see Definition \ref{func_der}). For $n \in \ZZ_+$, if there exists a locally integrable function $R^{(n)}_f(t,\cdot)$  such that 
\begin{align}
&\int_{[0,t]^n} ds_1 \cdots ds_n \frac{1}{n!} D^n F_t \big|_{\gamma=0} \phi(s_1) \cdots \phi(s_n) \nonumber \\
&= \int_{[0,t]^n} ds_1 \cdots ds_n  R^{(n)}_f(t,s_1,\dots,s_n) \phi(s_1)  \cdots \phi(s_n), \label{29}
\end{align}
for all test functions  $\phi \in C_c^\infty(0,t)$,   then $R_f^{(n)}(t,\cdot)$ is called the $n$th order response function of the observable $f$.
\end{defn}

Note that  since the derivatives in Eq. \eqref{29} are symmetric mappings (see Subsection \ref{app_diffcal}), the response functions $R_f^{(n)}$ are symmetric in $s_1, \dots, s_n$. 

We can gain some intuition on the response functions by first looking at $R_f^{(1)}$. Taking $\phi(s) = \delta(x-s) = \delta_s(x)$, we have, formally, $R_f^{(1)}(t,s) = \int_0^t R_f^{(1)}(t,u) \delta(s-u) du = \int_0^t du \frac{\delta F_t}{\delta \gamma}\big|_{\gamma=0} \delta(s-u) = \lim_{\epsilon \to 0} \frac{1}{\epsilon} (F_t[\epsilon \delta_s] - F_t[0])$. This tells us that $R_f^{(1)}$ represents the rate of change of the  functional $F_t$ at time $t$ subject to the small impulsive perturbation to $\vecc{h}$ at time $s$. Similarly, the $R_f^{(n)}$ give higher order rates of change for $n > 1$.  Summing up these rates of change allows us to quantify the full effect of the perturbation on the  functional  $F_t$ order by order.

We now show that, under certain assumptions, these rates of change are well-defined and compute them.  The following proposition provides  explicit expressions for the $n$th order response function for a class of observables of the SRNN. In the first order case, it was shown in \citep{chen2020mathematical} that the response function can be expressed as a correlation function of the observable and a unique conjugate observable with respect to the unperturbed dynamics, thereby providing  a mathematically rigorous version of the Agarwal-type fluctuation-dissipation relation (A-FDT). We are going to show that a similar statement can be drawn for the higher order cases.

In the following, let  $f \in C_b^\infty(\RR^n)$ be any observable and  $\Delta \mathcal{L}_t := \mathcal{L}_t - \mathcal{L}^0$, for $t \in [0,T]$.

\begin{prop} \label{exp_res} (Explicit expressions for  response functions)   For $n \in \ZZ_+$, let $R^{(n)}_f$ be the $n$th-order response function of $f$. Then, for $0 < s_n < s_{n-1} < \cdots < s_0 := t \leq T$:  \\
(a) \begin{align} 
    R_f^{(n)}(t,s_1,\dots,s_n) &= \mathbb{E}  P^0_{0,s_n} \Delta \mathcal{L}_{s_n} P^0_{s_n,s_{n-1}} \Delta \mathcal{L}_{s_{n-1}} \cdots P^0_{s_1,t} f(\vecc{h}_0).
\end{align}
(b) (Higher-order  A-FDTs)  If, in addition, $\rho_{init}$ is positive, then 
\begin{equation} 
R_f^{(n)}(t,s_1, \dots, s_n) = \mathbb{E}  f(\vecc{h}_0) v_{t,s_1,\dots, s_n}^{(n)}(\vecc{h}_{s_1}, \dots, \vecc{h}_{s_n} ), \label{33}
\end{equation} 
where 
\begin{equation}\label{con_obs}
v^{(n)}_{t,s_1,\dots, s_n}(\vecc{h}_1,\dots, \vecc{h}_n) = \frac{(-1)^n}{\rho_{init}}  (P^0_{s_1,t})^* \vecc{\nabla}_{\vecc{h}_1}^T[ \vecc{U}_{s_1} \cdots  (P^0_{s_n,s_{n-1}})^*  \vecc{\nabla}_{\vecc{h}_n}^T[ \vecc{U}_{s_n} p_{s_n}(\vecc{h})]].
\end{equation}
\end{prop}

We make a few remarks on the above results. 

\begin{rmk}
In the linear response regime with $\overline{\vecc{h}}$ stationary,  if one further restricts to reversible diffusion (when detailed balance holds), in which case  the drift coefficient in the SRNN can be expressed as negative gradient of a potential function, then  Eq. \eqref{33} reduces to the equilibrium FDT (see, for instance, \citep{Kubo66} or \citep{chetrite2008fluctuation}), a cornerstone of nonequilibrium statistical mechanics. In the one-dimensional case, explicit calculation and richer insight can be obtained. See, for instance, Example 9.5 in \citep{pavliotis2014stochastic} for a formal analysis of stochastic resonance using linear response theory.  
\end{rmk}

\begin{rmk}
The observables $v_{t,s_1,\dots,s_{n}}^{(n)}(\vecc{h}_{s_1}, \dots, \vecc{h}_{s_{n}})$ in Proposition \ref{exp_res}(b) can be viewed as nonlinear counterparts of the conjugate observable obtained in \citep{chen2020mathematical} in the linear response regime. Indeed, when $n=1$, $v_{t,s_1}^{(1)}(t,s_1)$ is exactly the conjugate observable in \citep{chen2020mathematical}. Hence, it is natural to call them higher order conjugate observables. Proposition \ref{exp_res}(b) tells us that any higher order response of the observable $f$ to a small input signal can be represented as  correlation function of $f$ and the associated higher order conjugate observable. 
\end{rmk}

Moreover, the conjugate observables $v_{t,s_1,\dots,s_n}^{(n)}$ are uniquely determined in the following sense.

\begin{cor} \label{unique}
Let $n \in \ZZ_+$ and $0 < s_n < \cdots < s_1 < s_0 := t \leq T$.
Assume that there is another function $\tilde{v}^{(n)}_{t,s_1,\dots,s_n} \in L^1(\rho_{init})$ on $\RR^{n} \times \cdots \times \RR^n$  such that
\begin{align}
\mathbb{E} f(\vecc{h}) v_{t,s_1,\dots,s_{n}}^{(n)}(\vecc{h}_{s_1}, \dots, \vecc{h}_{s_{n}})    =  \mathbb{E} f(\vecc{h}) \tilde{v}_{t,s_1,\dots,s_{n}}^{(n)}(\vecc{h}_{s_1}, \dots, \vecc{h}_{s_{n}}),    
\end{align}
for all $f \in C_c^\infty(\RR^n)$. Then $v^{(n)}_{t,s_1,\dots,s_{n}} = \tilde{v}^{(n)}_{t,s_1,\dots,s_{n}}$ almost everywhere.\\
\end{cor}

We now have the ingredients for the first main result addressing (Q1). 
The following theorem provides a series representation for the output functional of a SRNN driven by the deterministic  signal $\gamma$. It says that infinite series of the form \eqref{volt} are in fact (or can be made sense as)  Volterra series \citep{volterra1959theory, boyd1985fading} (see Subsection \ref{app_diffcal} for definition of Volterra series and related remarks).

\begin{thm} \label{thm1} (Memory representation)
Let $t \in [0,T]$. The output functional, $\mathbb{E} f(\vecc{h}_t)$, of the SRNN \eqref{e1}-\eqref{e3} is the limit as $N \to \infty$ of 
\begin{align}
F^{(N)}_t[\gamma] &=  \sum_{n=1}^{N} \int_{[0,t]^n} ds_1 \cdots ds_n R_f^{(n)}(t,s_1,\dots,s_n)  \gamma(s_1) \cdots \gamma(s_n),
\end{align}
where the $R_f^{(n)}$ are given in Proposition \ref{exp_res}. The limit exists and is a unique convergent  Volterra series. If $G_t$ is another such series  with the response functions $Q_f^{(n)}$, then $F_t = G_t$.
\end{thm}




Using Theorem \ref{thm1}  we can obtain another representation (c.f. Eq. \eqref{ultimate_main}) of the output functional, provided that additional (but reasonable) assumptions are imposed.  Recall that we are using Einstein's summation notation for repeated indices.

\begin{thm} \label{thm2} (Memoryless representation)  Assume that the  operator $\mathcal{A}^0$ admits  a well-defined eigenfunction expansion. Then, the output functional $\mathbb{E}f(\vecc{h}_t)$  of the SRNN \eqref{e1}-\eqref{e3} admits a convergent series expansion, which is the limit as $N,M \to \infty$ of:
\begin{align}
F^{(N,M)}_t[\gamma] &=   \sum_{n=1}^{N}   a_{p_0,\dots,p_n,l_1,\dots,l_n} t^{p_0} \int_0^t ds_1  s_1^{p_1} u^{l_1}_{s_1} \int_0^{s_1} ds_2 s_2^{p_2} u^{l_2}_{s_2}  \cdots \int_0^{s_{n-1}} ds_n s_n^{p_n} u^{l_n}_{s_n}, \label{memless}
\end{align}
where the $a_{p_0, \dots, p_n, l_1, \dots, l_n}$ are constant coefficients   that  depend on the $p_i$, the $l_i$, the eigenvalues and eigenfunctions of $\mathcal{A}^0$, $f$ and $\rho_{init}$ but independent of the input signal and time. Here, $p_i \in \{0,1,\dots,M\}$ and $l_i \in \{1,2,\dots,m\}$.      
\end{thm}

This representation abstracts away the details of memory induced by the SRNN  (thus the name memoryless) so that any parameters  defining the SRNN are encapsulated in the coefficients $a_{p_0,\dots,p_n,l_1,\dots,l_n}$ (see Eq. \eqref{exp_exp} for their explicit expression). The parameters are either learnable (e.g., learned via gradient descent using backpropagation) or fixed (e.g., as in the case of reservoir computing).  The representation tells us that the  iterated integrals in \eqref{memless} are building blocks used by SRNNs to extract the information in the input signal. See also Remark \ref{rmk_symmetric} and Remark \ref{rmk_visual} in {\bf SM}.  This result can be seen as a stochastic analog of the one in Chapter 3 of \citep{isidori2013nonlinear}. It can also be shown that Eq. \eqref{memless} is a generating series in the sense of Chen-Fliess \citep{fliess1981fonctionnelles}.


Composition of multiple SRNNs (i.e., a deep SRNN), preserves the form of the above two representations. More importantly, as in applications, composition of truncated (finite) series increases the number of nonlinear terms and gives a richer set of response functions and features in the resulting series representation. We make this precise for  composition of two SRNNs in the following proposition. Extension to composition of more than two SRNNs is straightforward.

\begin{prop}(Representations for certain deep SRNNs) Let $F_t$ and $G_t$ be the output functional of two SRNNs, with the associated truncated Volterra series having the response kernels $R_f^{(n)}$ and $R_g^{(m)}$, $n=1,\dots,N$, $m=1,\dots,M$, respectively. Then $(F \circ G)_t[\gamma] = F_t[G_t[\gamma]]$ is a truncated Volterra series with the $N+M$  response kernels:  
\begin{align}
&R^{(r)}_{fg}(t,t_1,\dots,t_r) \nonumber \\ 
&= \sum_{k=1}^r \sum_{\mathcal{C}_r} \int_{[0,t]^k} R^{(k)}_f(t,s_1, \dots, s_k) \nonumber \\ 
& \ \ \   \times R_g^{(i_1)}(t-s_1, t_1 - s_1, \dots, t_{i_1} - s_1) \cdots R_g^{(i_k)}(t-s_k, t_{r - i_k+1}-s_k, \dots, t_r - s_k) ds_1 \cdots ds_k, \label{bell!}
\end{align} 
for $r=1,\dots,N+M$, where  
\begin{equation}
\mathcal{C}_r = \{(i_1, \dots, i_k) :  i_1, \dots, i_k \geq 1, 1 \leq k \leq r,  i_1 + \dots + i_k = r\}.
\end{equation}
If $F_t$ and $G_t$ are Volterra series (i.e., $N,M=\infty$), then $(F \circ G)_t[\gamma]$ is a Volterra series (whenever it is well-defined) with the response kernels $R^{(r)}_{fg}$ above, for $r=1,2,\dots$. 

Moreover, the statements in Theorem \ref{thm2} apply to $(F \circ G)_t$, i.e., $(F \circ G)_t$ admits a convergent series expansion of the form \eqref{memless} under the assumption in Theorem \ref{thm2}. 
\end{prop}
Alternatively, the response kernels \eqref{bell!} can be expressed in terms of the exponential Bell polynomials \citep{bell1927partition} (see Subsection \ref{appendix_bell} in  {\bf SM}).

\begin{rmk}
It is the combinatorial structure and properties (such as convolution identity and recurrence relations)  of  the Bell polynomials that underlie the richness of the memory representation of a deep SRNN. To the best of our knowledge, this seems to be the first time where a connection between Bell polynomial and deep RNNs is made. It may be potentially fruitful to further explore this connection to study expressivity and memory capacity of different variants of deep RNNs \citep{pascanu2013construct}.
\end{rmk}

Next, we focus on (Q2). The key idea is to link the above representations for the output functional to the notion and properties of path signature  \citep{lyons2007differential, lyons2014rough,levin2013learning,liao2019learning}, by lifting  the input signal to a tensor algebra space.

Fix a Banach space $E$ in the following. Denote by 
\begin{equation}
T((E)) =  \{\vecc{a} :=  (a_0, a_1, \dots): \forall n \geq 0, \ a_n \in E^{\otimes n}\},
\end{equation}
where $E^{\otimes n}$ denotes the $n$-fold tensor product of $E$ ($E^{\otimes 0} := \RR$). This is the space of formal series of tensors of $E$ and can be identified with the free Fock space $T_0((E)) = \bigoplus_{n=0}^\infty E^{\otimes n}$ when $E$ is a Hilbert space \citep{parthasarathy2012introduction}. 

\begin{defn}(Signature of a path)  \label{def_sig}  Let $\vecc{X} \in C([0,T],  E)$ be a path of bounded variation (see Definition \ref{bv_def}). The {\it signature} of $\vecc{X}$ is the element $S$ of $T((E))$, defined as 
\begin{equation}
S(\vecc{X}) = (1,\vecc{X}^1,\vecc{X}^2,\dots),
\end{equation}
where
\begin{equation}
    \vecc{X}^n = \int_{\Delta^n_T} d\vecc{X}_{s_1} \otimes \dots \otimes d\vecc{X}_{s_n} \in E^{\otimes n},
\end{equation}
for $n\in \ZZ_+$, with  $\Delta^n_T := \{0\leq s_1\leq \dots \leq s_n \leq T \}$.
\end{defn}

Let $(e_{i_1} \otimes \dots \otimes e_{i_n})_{(i_1, \dots, i_n) \in \{1,\dots,m\}^n}$ be the canonical basis of $E^{\otimes n}$, then we have:
\begin{equation}
    S(\vecc{X}) = 1 + \sum_{n =1}^\infty \sum_{i_1, \dots, i_n} \left( \int_{\Delta^n_T} dX^{i_1}_{s_1} \cdots dX^{i_n}_{s_n} \right) (e_{i_1} \otimes  \dots \otimes e_{i_n}) \in T((E)).   
\end{equation}
Denoting by $\langle \cdot, \cdot \rangle$ the dual pairing, we have 
\begin{equation}
\langle S(\vecc{X}), e_{i_1} \otimes  \dots \otimes e_{i_n} \rangle  = \int_{\Delta_T^n} dX_{s_1}^{i_1} \cdots dX_{s_n}^{i_n}. 
\end{equation} 

\begin{thm}(Memoryless representation in terms of signature) \label{sigrep}
Let $p$ be a  positive integer and assume that the input signal $\vecc{u}$ is a path of bounded variation. Then the output functional $F_t$ of a SRNN is the limit as $p \to \infty$ of $F^{(p)}_t$, which are linear functionals of the signature of the path, $\vecc{X}^{(p)} = \vecc{u} \otimes \vecc{\psi}^{(p)} \in \RR^{m \times p}$ (which can be identified with $\RR^{mp}$ via vectorizaton), where $\vecc{\psi}^{(p)} = (1,t,t^2,\dots,t^{p-1}) \in \RR^p$, i.e., 
\begin{align}
    F_t^{(p)}[\vecc{u}] =  \sum_n b_n(t) \langle S(\vecc{X}_t^{(p)}), e_{i_1} \otimes \dots \otimes e_{i_n} \rangle,
\end{align}
where the $b_n(t)$ are coefficients that only depend on $t$.
\end{thm}

It was shown in \citep{lyons2014rough} that the signature is a universal feature set, in the sense that any continuous map can be approximated by a linear functional of signature of the input signal (see also Remark \ref{rmk_a3}). This result is a direct consequence of the Stone-Weierstrass theorem. On the other hand, Theorem \ref{sigrep} implies that the output functional of SRNNs admits a  linear representation in terms of signature of a time-augmented input signal, a lift of the original input signal to higher dimension to account for time. We shall call the signature of this time-augmented input signal the {\it response feature}, which is richer as a feature set than the signature of the input signal and may  be incorporated in a machine learning framework for learning sequential data.

\begin{rmk}
Note  that our SRNN \eqref{e1}-\eqref{e3}  can be interpreted as a controlled differential equation (CDE) \citep{lyons2007differential}.
We emphasize that while the connection between signature and CDEs are well known within the rough paths community, it is  established explicitly using  local Taylor approximations (see, for instance, \citep{boedihardjo2015uniform} or Section 4 of \citep{liao2019learning}) that ignore any (non-local) memory effects, whereas here we establish such connection explicitly by deriving a signature-based representation from the memory representation that takes into account the memory effects (albeit under more restrictive assumptions).  It is precisely because of the global (in time) nature of our approximation that the resulting representation for our particular output functionals is in terms of signature of a time-augmented signal and not simply in terms of signature of the signal itself. Our response theory based approach offers alternative, arguably more intuitive, perspectives on how signature arise in SRNNs. This may be useful for readers  not familiar with rough paths theory.  
\end{rmk}



\subsection{Formulating SRNNs as Kernel Machines} 

We now consider a supervised learning (regression or classification) setting where we are given $N$ training  input-output pairs $(\vecc{u}_n, y_n)_{n=1,\dots,N}$, where the $\vecc{u}_n \in \mathcal{\chi}$,  the space of paths in $C([0,T],\RR^m)$ with bounded variation, and $y_n \in \RR$, such that $y_n = F_{T}[\vecc{u}_n]$ for all $n$. Here  $F_T$ is a continuous target mapping. 

Consider the optimization problem:
\begin{equation}
\min_{\hat{F} \in \mathcal{G} } \{L(\{(\vecc{u}_n,y_n,\hat{F}[\vecc{u}_n])\}_{n=1,\dots,N}) + R(\|\hat{F}\|_{\mathcal{G}})\}, \label{opt_p}
\end{equation}
where $\mathcal{G}$ is a hypothesis (Banach) space with norm $\| \cdot\|_{\mathcal{G}}$, $L:(\chi\times \RR^2)^N \to \RR \cup \{\infty\}$ is a loss functional and $R(x)$ is a strictly increasing real-valued function in $x$.  

Inspired by Theorem \ref{sigrep} (viewing $\mathcal{G}$ as a hypothesis space induced by the SRNNs), we are going to show that the solution to this problem  can be expressed as kernel expansion over the training examples (c.f. \citep{evgeniou2000regularization,scholkopf2002learning,hofmann2008kernel}). 

In the following, consider the Hilbert space 
\begin{equation}
    \mathcal{H} := \mathcal{P} \otimes \RR^m,
\end{equation}
where $\mathcal{P}$ is the appropriately weighted $l^2$ space of sequences of the form $(P_0(t), P_1(t), \cdots  )$ with the $P_n(t)$ orthogonal polynomials on $[0,T]$. Let $T_s((\mathcal{H}))$ denote the symmetric Fock space over $\mathcal{H}$ (see Subsection \ref{B9} for definition) and $T_s^{\otimes L}((\mathcal{H}))$ denote  $L$-fold tensor product of $T_s((\mathcal{H}))$ for $L \in \ZZ_{+}$. 

\begin{prop} \label{rk}
Let $L \in \ZZ_{+}$. Consider the map $K: \mathcal{H} \times \mathcal{H} \to \RR$, defined by 
\begin{equation}
K(\vecc{v},\vecc{w}) = \langle S(\vecc{v}), S(\vecc{w})  \rangle_{T^{\otimes L}_s((\mathcal{H}))}.
\end{equation}
Then $K$ is a kernel over $\mathcal{H}$ and there exists a unique  RKHS, denoted $\mathcal{R}_L$ with the norm $\| \cdot \|_{\mathcal{R}_L}$, for which  $K$ is a reproducing kernel.  
\end{prop}

One can view SRNNs, with only the weights in the readout layer optimized and the weights in the hidden layer kept fixed and not optimized, as  kernel machines operating on a RKHS  associated with the response feature. In particular, we can see the  continuous-time version of echo state networks in reservoir computing \citep{lukovsevivcius2009reservoir} as a kernel method. This  is captured precisely in the following theorem.





\begin{thm}(Representer theorem)  \label{rep}
Consider the time-augmented paths $\vecc{X}_n =   \vecc{v} \otimes \vecc{u}_n $, where the $\vecc{u}_n$ are $\RR^m$-valued  input paths in $\chi$ and $\vecc{v}$ is a $\RR^{\ZZ_{+}}$-valued vector in $ \mathcal{P}$. Then: \\ 
(a) Any solution to the optimization problem \eqref{opt_p} with the hypothesis space of $\mathcal{G} := \mathcal{R}_1$ admits a representation of the form:
\begin{equation}\label{45}
\hat{F}^*_t =  \sum_{n=1}^N c_n \langle S(\vecc{X}_n), S(\vecc{X})  \rangle_{T_s((\mathcal{H}))},
\end{equation}
where the $c_n \in \RR$ and $N$ is the number of training input-output pairs. \\

\noindent (b) Let $L \in \ZZ_+$. If we instead consider the paths, denoted $\tilde{\vecc{X}}_n$, obtained by linear interpolating on the $L+1$ data points $(\vecc{X}_n(t_i))_{i = 0,1,\dots,L}$ sampled at time $t_i \in [0,T]$, then  any solution to the corresponding optimization problem \eqref{opt_p} with the hypothesis space of $\mathcal{G} := \mathcal{R}_L$ admits a representation of the form:
\begin{equation}\label{46}
\hat{F}^*_t = \sum_{n=1}^N  \alpha_n  \prod_{l=1}^L  \exp\left(\left\langle \Delta \vecc{X}_n^{(l)}, \Delta  \vecc{X}^{(l)} \right\rangle _{\mathcal{H}} \right), 
\end{equation}
where the $\alpha_n \in \RR$ and $\Delta \vecc{X}^{(l)} := \vecc{X}(t_l) - \vecc{X}(t_{l-1})$ for $l=1,\dots,L$.

\end{thm}



We emphasize that the idea of representing Volterra series as elements of a RKHS is certainly not new (see, e.g., \citep{zyla1983nonlinear,franz2006unifying}). However, it is the use of the  response feature here that makes our construction differs from those in previous works (c.f. \citep{kiraly2019kernels,toth2020bayesian}). The representer theorem here is obtained for a wide class of optimization problems using an alternative approach in the SRNN setting. The significance of the theorem obtained here is not only that the optimal solution will live in a subspace with dimension no greater than the number of training examples but also how the solution depends on the number of samples in the training data stream.  Note  that the appearance of orthogonal polynomials here is not too surprising given their connection to nonlinear dynamical systems \citep{kowalski1997nonlinear,mauroy2020koopman}. Lastly, we remark that the precise connections between (finite width) RNNs  whose all weights are optimized (via gradient descent) in deep learning and kernel methods remain an open problem and we shall leave it to future work.


\section{Conclusion}
\label{concl}
In this paper we have addressed two fundamental questions concerning a class of stochastic recurrent neural networks (SRNNs), which can be models for  artificial or biological networks, using the nonlinear response theory from nonequilibrium statistical mechanics as the starting point. In particular, we are able to characterize, in a systematic and order-by-order manner, the response of the SRNNs to a perturbing deterministic input signal by deriving two types of series representation for the output functional of these SRNNs and a deep variant in terms of the driving input signal.  This provides insights into the nature of both memory and memoryless representation induced by these driven networks.  Moreover, by relating these representations to the notion of path signature, we find that the  set of response feature is the  building block in which SRNNs extract information from when processing an input signal, uncovering the universal mechanism underlying the operation of SRNNs. In particular, we have shown, via a representer theorem, that SRNNs can be viewed as kernel machines operating on a reproducing kernel Hilbert space associated with the response feature.

We end with a few final remarks. From mathematical point of view, it would be interesting to relax the assumptions here and work in a general setting where the driving input signal is a rough path, where regularity of the input signal could play an important role. One could also study how SRNNs respond to perturbations in the input signal and in the driving noise (regularization) by adapting the techniques developed here. So far we have focused on the ``formulate first" approach mentioned in Introduction. The results obtained here suggest that one could study the ``discretize next" part by devising efficient algorithm to exploit the use of discretized response features  and related features in machine learning tasks involving temporal data, such as predicting time series generated by complex dynamical systems arising in science and engineering.

\acks{The author is grateful to the support provided by the Nordita Fellowship 2018-2021.}

\newpage

\appendix 
\section*{\centering \fontsize{15}{15} \bf Supplementary Material ({\bf SM})}



\section{Preliminaries and Mathematical Formulation}
\label{main}




\subsection{Differential Calculus on Banach Spaces} \label{app_diffcal}

In this subsection, we present elements of differential calculus on Banach spaces, introducing our notation and terminology along the way. We refer to the classic book of Cartan \citep{cartan1983differential} for more details (see also \citep{abraham2012manifolds}). 

 
We will need the notion of functional derivatives before we dive into response functions.
Functional derivatives are generalization of ordinary derivatives to functionals. At a formal level, they can be defined via the variation $\delta F$ of the functional $F[u]$ which results from variation of $u$ by $\delta u $, i.e., $\delta F= F[u+\delta u] - F[u]$. The technique used to evaluate $\delta F$ is a Taylor expansion of the functional
$F[ u +\delta u ] = F[ u + \epsilon \eta ]$ in powers of $\delta u$ or  $\epsilon$. The functional $F[ u +\epsilon \eta]$ is an ordinary function of $\epsilon$. This implies that the expansion in terms of powers of $\epsilon$
is a standard Taylor expansion, i.e.,
\begin{equation}
F[u+\epsilon \eta] = F[u] + \frac{dF[u+\epsilon \eta]}{d\epsilon}\bigg|_{\epsilon = 0} \epsilon + \frac{1}{2!} \frac{d^2F[u+\epsilon \eta]}{d\epsilon^2}\bigg|_{\epsilon = 0} \epsilon^2 + \cdots,
\end{equation}
provided that the ``derivatives'' above can be made sense of. We first define such ``derivatives".
 
Recall that for a function $f$ of a real variable, the derivative of $f$ is defined by 
\begin{equation}
    f'(x) = \lim_{h \to 0} \frac{f(x+h)-f(x)}{h},
\end{equation}
provided that the limit exists. This definition becomes absolete when $f$ is a function of vector variable since then division by a vector is meaningless. Therefore, to define a derivative for mappings from Banach space into a Banach space, one needs to revise the above definition. This leads to the notion of Fr\'echet differentiability, generalizing the notion of slope of the line tangent to the graph of the function at some point.

In the following, let $E_1$, $E_2$ be Banach spaces over $\RR$, $T:E_1 \to E_2$ be a given mapping, and $\|\cdot\|$ represents the norm on appropriate space.

\begin{defn} (Fr\'echet differentiabiliy) Fix an open  subset $U$ of a Banach space $E_1$ and let $u_0 \in U$.
We say that the mapping $T: U \to E_2$ is Fr\'echet differentiable at the point $u_0$ if there is a bounded linear map  $DT(u_0) :E_1 \to E_2$ such that for every $\epsilon >0$, there is a $\delta > 0$ such that 
\begin{equation}
    \frac{\|T(u_0 + h)-T(u_0) - DT(u_0) \cdot h \|}{\|h\|} < \epsilon
\end{equation}
whenever $\|h\| \in (0,\delta)$, where $DT(u_0) \cdot e$ denotes the evaluation of $DT(u_0)$ on $e \in E_1$. This can also be written as 
\begin{equation}
    \lim_{\|h\|\to 0} \frac{\|T(u_0 + h)-T(u_0)-DT(u_0) \cdot h\|}{\|h\|} = 0.
\end{equation}
Note that this is equivalent to the existence of a linear map $DT(u_0) \in L(E_1, E_2)$ such that
\begin{equation}
    T(u_0 + h) - T(u_0) = DT(u_0) \cdot h + e(h)
\end{equation}
where 
\begin{equation}
    \lim_{\|h\| \to 0} \frac{\|e(h)\|}{\|h\|} = 0,
\end{equation}
i.e., $e(h) = o(\|h\|)$ as $\|h\| \to 0$.
\end{defn}

\begin{defn}(Fr\'echet derivative) If the mapping $T$ is Fr\'echet differentiable at each $u_0 \in U$, the map $DT: U \to L(E_1, E_2)$, $u \mapsto DT(u)$, is called the (first order) Fr\'echet derivative of $T$. Moreover, if $DT$ is a norm continuous map, we say that $T$ is of class $C^1$. We define, inductively, for $r \geq 2$, the $r$-th order derivative $D^r T := D(D^{r-1} T) : U \subset E_1 \to L^{(r)}(E_1, E_2):=L(E_1, L^{(r-1)}(E_1,E_2))$, with $L^{(1)}(E_1,E_2) := L(E_1, E_2)$, whenever it exists. If $D^r T$ exists and is norm continuous, we say that $T$ is of class $C^r$. 
\end{defn}

A few remarks follow. A weaker notion of differentiability, generalizing the idea of directional derivative in finite dimensional spaces, is provided by Gateaux. 

\begin{defn}(Gateaux differentiability and derivative)
Fix $u_0 \in E_1$. The mapping $T: E_1 \to E_2$ is Gateaux differentiable at $u_0$ if there exists a continuous linear operator $A$ such that 
\begin{equation}
    \lim_{\epsilon \to 0} \left\| \frac{T(u_0+\epsilon h)-T(u_0)}{\epsilon} - A(h) \right\| = 0
\end{equation}
for every $h \in E_1$, where $\epsilon \to 0$ in $\RR$. The operator $A$ is called the Gateaux derivative of $T$ at $u_0$ and its value at $h$ is denoted by $A(h)$. The higher order Gateaux derivatives can be defined by proceeding inductively.
\end{defn}
 
\begin{rmk}
It is a standard result that if the Gateaux derivative exists, then it is unique (similarly for Fr\'echet derivative). Gateaux differentiability is weaker than that of Fr\'echet in the sense that if a mapping has the Fr\'echet derivative at a point $u_0$, then it has the Gateaux derivative at $u_0$ and both derivatives are equal, in which case $A(h) = DT(u_0) \cdot h$ using the notations in the above definitions. The converse does not generally hold. For basic properties of these derivatives, we refer to the earlier references. From now on, we will work with the more general notion of Fr\'echet differentiability.
\end{rmk}

We now introduce the notion of functional derivative for the special case of mappings that are real-valued functionals.  Let $E_2 = \RR$, so that $T$ is a functional on $E_1$. When $T$ is Fr\'echet differentiable at some $u \in E_1$,  its derivative is a bounded linear functional on $E_1$, i.e., $DT[u] \in L(E_1, \RR) =: E_1^*$. If $E_1$ is a Hilbert space, then by the Riesz representation theorem\footnote{Note that we cannot apply this fact when defining the response functions in Definition \ref{def_res} since the space of test functions is not a Hilbert space.}, there exists a unique element $y \in E_1$ such that $DT[u] \cdot e = \langle y, e \rangle$ for every $e \in E_1$. The derivative $DT[u]$ can thus be identified with $y$, which we sometimes call the gradient of $T$. In the case when $E_1$ is a functional space, we call the derivative the {\it functional derivative} of $T$ with respect to $u$, denoted $\delta T/\delta u$:
\begin{equation}
    \langle \delta T/\delta u, e \rangle = DT[u] \cdot e.
\end{equation}  
Moreover, $T$ is also Gateaux differentiable at $u \in E_1$ and we have $DT[u] \cdot e = \frac{d}{dt}\big|_{t=0} T[u+te]$. 

The higher order derivatives can be defined inductively. For instance, if $T$ is a twice differentiable functional on a real Hilbert space $E_1$, then for all $u \in E_1$, $D^2 T[u] \cdot (e_1,e_2) = D((DT)(\cdot) \cdot e_2)[u] \cdot e_1$.  More generally, if $e_1, \dots, e_{n-1} \in E_1$, $T: U \to E_2$ is $n$ times differentiable, then the map $Q:U \to E_2$ defined by $Q[u] = D^{n-1} T[u] \cdot (e_1, \dots, e_{n-1})$ is differentiable and $DQ[u] \cdot e = D^n T[u] \cdot (e,e_1, \dots, e_{n-1})$. Moreover, $D^n T[u]$ is multi-linear symmetric, i.e., 
\begin{equation} \label{deriv_sym}
D^n T[u] \cdot (e_0,e_1,\dots,e_{n-1}) = D^n T[u] \cdot (e_{\sigma(0)}, e_{\sigma(1)}, \dots, e_{\sigma(n-1)}),
\end{equation}
where $\sigma$ is any permutation of $\{0,1,\dots,n\}$ (see Theorem 5.3.1 in \citep{cartan1983differential}). In this case, we have
\begin{equation}
    \lim_{\epsilon \to 0} \left\| \frac{(D^{n-1} T[u+\epsilon h] - D^{n-1}T[u])\cdot (h_1, \dots, h_{n-1})}{\epsilon} - D^n T[u] \cdot (h, h_1, \dots, h_{n-1}) \right\| = 0
\end{equation}
and we write 
\begin{equation}
    D^n T[u] \cdot (h, h_1, \dots, h_{n-1}) = \lim_{\epsilon \to 0}  \frac{(D^{n-1} T[u+\epsilon h] - D^{n-1}T[u])\cdot (h_1, \dots, h_{n-1})}{\epsilon}.
\end{equation}

The notion of functional derivatives that we will need is defined in the following.


\begin{defn} (Functional derivatives) \label{func_der}
Fix $t>0$. Let $F$ be a functional on $C([0,t],\RR)$ and $u \in C([0,t],\RR)$. Then for $n \in \ZZ_{+}$, the $n$th order functional derivative of $F$ with respect to $u$ is a functional $\delta^n F/\delta \vecc{u}^{(n)}$, with $\delta \vecc{u}^{(n)} := \delta  u(s_1) \cdots \delta u(s_n)$, on $C_c^\infty(0,t)$, defined as:
\begin{align}
\int_{(0,t)^n}  ds_1 \cdots ds_n \frac{\delta^n F[u]}{\delta \vecc{u}^{(n)}} \phi(s_1) \cdots \phi(s_n) = D^n F[u] \cdot (\phi,\phi,\dots,\phi), 
\end{align}
whenever the derivative exists, for any  $\phi \in C_c^\infty(0,t)$.
\end{defn}

Recall the following Taylor's theorem for a mapping $T: U \to E_2$.

\begin{thm} (Taylor's theorem -- Theorem 5.6.3 in \citep{cartan1983differential}) 
Let $T$ be an $n-1$ times differentiable mapping. Suppose that $T$ is $n$ times differentiable at the point $u \in U$. Then, denoting $(h)^n := (h, \dots, h)$ (n times),
\begin{equation}
    \left\|T(u+h) - T(u) - DT(u) \cdot h - \cdots - \frac{1}{n!} D^n T(u) \cdot (h)^n  \right\| = o(\|h\|^n). 
\end{equation}
\end{thm}

By {\it Taylor series} of a mapping $T$ at a point $u \in U$, we mean the series of homogeneous polynomials given by
\begin{equation}
    T(u+h) = T(u) + DT(u) \cdot h  + \cdots + \frac{1}{n!} D^n T(u) \cdot (h)^n  + \dots
\end{equation}
that is absolutely convergent.

An important example of Taylor series of a mapping is given by the Volterra series, which is widely studied in systems and control theory \citep{boyd1985fading,brockett1976volterra,fliess2012algebraic}.

\begin{defn} (Volterra series \citep{boyd1984analytical}) \label{def_vol}
Let $u \in C([0,t],\RR)$. A Volterra series operator is an operator $N$ given by:
\begin{equation}
    N_t[u] = h_0 + \sum_{n=1}^\infty \int_{[0,t]^n} h_n(s_1, \dots, s_n) u(t-s_1) \cdots u(t-s_n) ds_1 \cdots ds_n,
\end{equation}
satisfying $h_0 \in \RR$, $h_n \in B(\RR_{+}^n)$, and $a :=\limsup_{n \to \infty} \|h_n\|_\infty^{1/n} < \infty$, i.e., $\{\|h_n\|_\infty^{1/n}\}_{n \in \ZZ_+}$ is bounded. 
\end{defn}
It can be shown that the integrals and sum above converge absolutely for inputs with $\|u\|_\infty < \rho := 1/a$, i.e., for $u$ belonging to the ball of radius $\rho$, denoted $B_\rho$, in $L^\infty$. Moreover, $\|N_t[u]\|_\infty \leq g(\|u\|_\infty) := |h_0| + \sum_{n=1}^\infty \|h_n\|_\infty \|u\|_\infty^n$. Also, $N_t$ is a continuous map from $B_\rho$ into $L^\infty$  and has bounded continuous Fr\'echet derivatives of all orders, with \citep{boyd1985fading}
\begin{align}
    &D^{k}N_t[u] \cdot (u_1, \dots, u_k) \nonumber \\
    &= \sum_{n=k}^\infty n(n-1) \dots (n-k+1) \int_{[0,t]^n} SYM h_n(s_1, \dots, s_n) \prod_{i=1}^k u_i(t-s_i) ds_i \prod_{i=k+1}^n u(t-s_i) ds_i, 
\end{align}
where $SYM h_n(s_1, \dots, s_n) := \frac{1}{n!} \sum_{\sigma \in S_n} h_n(s_{\sigma(1)}, \dots, s_{\sigma(n)})$, with $S_n$ the group of all permutations of the set $\{1,\dots,n\}$.  

When evaluated at $u=0$, these derivatives can be associated with the $n$th term of the Volterra series and are given by:
\begin{equation} \label{62}
    \frac{1}{n!} D^n N_t[0] \cdot (u_1, \dots, u_n) = \int_{[0,t]^n} h_n(s_1, \dots, s_n) u_1(t-s_1) \cdots u_n(t-s_n) ds_1 \cdots ds_n.
\end{equation}
Therefore, the Volterra series above are in fact Taylor series of operators from $L^\infty$ to $L^\infty$. 

\subsection{Signature of  a Path}

We provide some background on the signature of a path. The signature is an object arising in the rough paths theory, which provide an elegant yet robust nonlinear extension of the classical theory of differential equations driven by irregular signals (such as the Brownian paths). In particular, the theory allows for deterministic treatment of stochastic differential equations (SDEs). 

For our purpose, we do not need the full machinery of the theory but only a very special case. We refer to  \citep{lyons2002system,lyons2007differential,friz2010multidimensional,friz2014course} for full details.  In particular, we will only consider signature for paths with bounded variation. The following materials are borrowed from \citep{lyons2007differential}. 

Fix a Banach space $E$ (with norm $| \cdot |$) in the following. We first recall the notion of paths with bounded variation.

\begin{defn}(Path of bounded variation) \label{bv_def}
Let $I = [0,T]$ be an interval. A continuous path $\vecc{X}: I \to E$ is of bounded variation (or has finite 1-variation) if 
\begin{equation}
    \|\vecc{X}\|_{1-\text{var}} := \sup_{\mathcal{D} \subset I} \sum_{i=0}^{n-1}  |\vecc{X}_{t_{i+1}}-\vecc{X}_{t_i}| < \infty, \label{sum}
\end{equation}
where $\mathcal{D}$ is any finite partition of $I$, i.e., an increasing sequence $ (t_0, t_1, \dots, t_n)$ such that $0 \leq t_0 < t_1 < \cdots < t_n \leq T$. 
\end{defn}

When equipped with the norm $\|\vecc{X} \|_{BV}  := \|\vecc{X}\|_{1-\text{var}} +  \|\vecc{X}\|_\infty$, the space of continuous paths of bounded variation with values in $E$ becomes a Banach space.

The following lemma will be useful later. 
\begin{lem} \label{bv}
Let $E = \RR$. If the path $X$ is increasing on $[0,T]$, then $X$ is of bounded variation on $[0,T]$. Moreover, if $Y$ and $Z$ are paths of bounded variation on $[0,T]$ and $k$ is a constant, then $X+Y$, $X-Y$, $XY$ and $kX$ are paths of bounded variation on $[0,T]$. 
\end{lem}
\begin{proof}
    The first statement is an easy consequence of the telescoping nature of the sum in \eqref{sum}  for any finite partitions of $[0,T]$. In fact, $\|X\|_{1-\text{var}} = X_T-X_0$. For the proof for the last statement, we refer   to Theorem 2.4 in  \citep{grady2009functions}.
\end{proof}

The signature of a continuous path with bounded variation lives in some tensor algebra  space, which we now elaborate on. 

Denote by $T((E))  =  \{\vecc{a} :=  (a_0, a_1, \dots): \forall n \in \NN, \ a_n \in E^{\otimes n}\}$ a tensor algebra space, where $E^{\otimes n}$ denotes the $n$-fold tensor product of $E$ ($E^{\otimes 0} := \RR$), which can be identified with the space of homogeneous noncommuting polynomials of degree $n$.  $T((E))$ is the space of formal series of tensors of $E$ and, when endowed with suitable operations, becomes a real non-commutative algebra (with unit $\Omega = (1,0,0,\dots)$). When $E$ is a Hilbert space, it can be identified with the free Fock space, $T_0((E)) := \bigoplus_{n=0}^{\infty} E^{\otimes n}$  \citep{parthasarathy2012introduction}. 

As an example/reminder, consider the case when $E = \RR^m$. In this case, if $(e_1, \dots, e_m)$ is a basis of $E$, then any element $ \vecc{x}_n \in E^{\otimes n}$ can be written as the expansion:
\begin{equation}
    \vecc{x}_n = \sum_{I = (i_1, \dots, i_n) \subset \{1,\dots, m\}^n} a_{I} (e_{i_1} \otimes \cdots \otimes e_{i_n})
\end{equation}
where the $a_I$ are scalar coefficients. For any nonnegative integer $n$, the tensor space $E^{\otimes n}$ can be endowed with the inner product $\langle \cdot, \cdot \rangle_{E^{\otimes n}}$ (defined in the usual way \citep{parthasarathy2012introduction}). Then,  for $\vecc{a} = (a_0, a_1, \dots)$ and $\vecc{b} = (b_0, b_1, \dots)$ in  $T((E))$, we can define an inner product in $T((E))$ by:
\begin{equation}\langle \vecc{a}, \vecc{b} \rangle_{T((E))} = \sum_{n \geq 0} \langle a_n, b_n \rangle_{E^{\otimes n}}.
\end{equation}

We  now define the signature of a path as an element in $T((E))$. 

\begin{defn}(Definition \ref{def_sig} in the main paper)   Let $\vecc{X} \in C([0,T],  E)$ be a path of bounded variation. The {\it signature} of $\vecc{X}$ is the element $S$ of $T((E))$, defined as 
\begin{equation} 
S(\vecc{X}) = (1,\vecc{X}^1,\vecc{X}^2,\dots),
\end{equation}
where, for $n\geq 1$,  $\Delta^n_T := \{0\leq s_1\leq \dots \leq s_n \leq T \}$, and
\begin{equation}
    \vecc{X}^n = \int_{\Delta^n_T} d\vecc{X}_{s_1} \otimes \dots \otimes d\vecc{X}_{s_n} \in E^{\otimes n}.
\end{equation}
\end{defn}

Note that in the definition above since $\vecc{X}$ is a path of bounded variation, the integrals are well-defined as Riemann-Stieljes integrals. The signature in fact determines the path completely (up to tree-like equivalence); see Theorem 2.29 in \citep{lyons2007differential}.

Let $(e_{i_1} \otimes \dots \otimes e_{i_n})_{(i_1, \dots, i_n) \in \{1,\dots,m\}^n}$ be the canonical basis of $E^{\otimes n}$, then we have:
\begin{equation}
    S(\vecc{X}) = 1 + \sum_{n =1}^\infty \sum_{i_1, \dots, i_n} \left( \int_{\Delta^n_T} dX^{i_1}_{s_1} \cdots dX^{i_n}_{s_n} \right) (e_{i_1} \otimes  \dots \otimes e_{i_n}) \in T((E)).   
\end{equation}


Denoting by $\langle \cdot, \cdot \rangle$ the dual pairing, we have $\langle S(\vecc{X}), e_{i_1} \otimes  \dots \otimes e_{i_n} \rangle  = \int_{\Delta_T^n} dX_{s_1}^{i_1} \cdots dX_{s_n}^{i_n}$.  \\

Next, we discuss a few special cases where the signature can be computed explicitly.  If $\vecc{X} := X$ is a one-dimensional path, then we can compute  the $n$th level signature to be  $\vecc{X}^n = (X_T - X_0)^n/n!$. If $\vecc{X}$ is an $E$-valued linear path, then $\vecc{X}^n = (\vecc{X}_T-\vecc{X}_0)^{\otimes n}/n!$ and $S(\vecc{X}) = \exp((\vecc{X}_T -\vecc{X}_0))$.  

We will need the following fact, which is useful in itself for numerical computation in practice, later.

\begin{lem} \label{piece}
If $\vecc{X}$ is a $E$-valued piecewise linear path, i.e., $\vecc{X}$ is obtained by concatenating $L$  linear paths, $\vecc{X}^{(1)}, \dots, \vecc{X}^{(L)}$, such that $\vecc{X} = \vecc{X}^{(1)} \star \cdots \star \vecc{X}^{(L)}$ (with $\star$ denoting path concatenation),  then:
\begin{equation}
S(\vecc{X}) = \bigotimes_{l=1}^L S(\vecc{X}^{(l)}) =  \bigotimes_{l=1}^L \exp{(\Delta \vecc{X}^{(l)})},
\end{equation}
where the $\Delta \vecc{X}^{(l)}$ denotes the increment of the path $\vecc{X}^{(l)}$ between its endpoints.
\end{lem}
\begin{proof}
The lemma follows from applying  Chen's identity (see Theorem 2.9 in \citep{lyons2007differential}) iteratively and exploiting linearity of the paths $\vecc{X}^{(l)}$.
\end{proof}

A few remarks are now  in order.

\begin{rmk} By definition, the signature is  a collection of definite iterated integrals of the path \citep{chen1977iterated}. It  appears naturally as the basis to represent the solution to a linear controlled differential equation via  the Picard iteration (c.f. Lemma 2.10 in \citep{lyons2007differential}).  The  signature of the path is in fact involved in a lot of rich algebraic and geometric structures.  In particular, the first level signature, $\vecc{X}^1$, is the increment of the path, i.e., $\vecc{X}_T - \vecc{X}_0$, the second level signature, $\vecc{X}^2$, represents the signed area enclosed by the path and the cord connecting the ending and starting point of the path. For further properties of the signature, we refer to the earlier references. Interestingly, iterated integrals are also objects of interest in quantum field theory and renormalization \citep{Brown2013qva,kreimer1999chen}.
\end{rmk}

\begin{rmk}\label{rmk_a3}
From the point of view of machine learning, the signature is an efficient summary of the information contained in the path and leads to invention of the Signature Method for ordered data (see, for instance, \citep{che2016primer} for a tutorial and the related works by the group of Terry Lyons). More importantly, the signature is  a universal feature, in the sense that any continuous functionals on compact sets of paths can be approximated arbitrarily well by a linear functional on the signature. This is stated precisely in Theorem 3.1 in \citep{levin2013learning} (see also Theorem B.1 in \citep{morrill2020neural} with log-signature as the feature instead).
\end{rmk}

\section{Proof of Main Results and Further Remarks}

\subsection{Auxiliary Lemmas}
Recall that we assume Assumption \ref{imp_ass} to hold 
throughout the paper unless stated otherwise. The following lemma on boundedness and continuity of amplitude of the input perturbation will be essential.

\begin{lem}  \label{B2}
Let $\gamma = (\gamma(t))_{t \in [0,T]}$, i.e., the path measuring the amplitude of $\vecc{C} \vecc{u}_t$ driving the SRNN. Then $\gamma \in C([0,T],\RR)$ and thus bounded.  
\end{lem}
\begin{proof}
Note that using boundedness of $\vecc{C}$, for $0 \leq s < t \leq T$:
\begin{align}
| |\vecc{C} \vecc{u}_t | - |\vecc{C} \vecc{u}_s | | &\leq |\vecc{C}\vecc{u}_t - \vecc{C} \vecc{u}_s | \leq |\vecc{C}(\vecc{u}_t - \vecc{u}_s)|  \leq C |\vecc{u}_t - \vecc{u}_s |,
\end{align}
where $C > 0$ is a  contant. Therefore, continuity of $\gamma$ follows from continuity of $\vecc{u}$.  That $\gamma$ is bounded follows from the fact that continuous functions on compact sets are bounded.
\end{proof}

In the sequel, for two sets $A$ and $B$, $A \backslash B$ denotes the set difference of $A$ and $B$, i.e., the set of elements in $A$ but not in $B$. 

The following lemma will be useful later.

\begin{lem} \label{lem_teles}
Let  $a_i, b_i$ be operators, for  $i = 1,\dots,N$.  Then, for $N \geq 2$,\\
(a) 
\begin{equation}
    \prod_{n=1}^N a_n - \prod_{m=1}^N b_m = \sum_{k=1}^N \left(\prod_{l=1}^{k-1} a_l\right) (a_k - b_k) \left(\prod_{p=k+1}^N b_{p}\right), \label{telescope}
\end{equation} 
(b) \begin{align}
    &\prod_{n=1}^N a_n - \prod_{m=1}^N b_m \nonumber \\ 
    &= \sum_{k=1}^N \left(\prod_{l=1}^{k-1} b_l\right) (a_k - b_k) \left(\prod_{p=k+1}^N b_{p}\right)  + \sum_{k=1}^N  \bigg( \sum_{p_1, \dots, p_{k-1} \in \Omega}  b_1^{p_1} (a_1 - b_1)^{1-p_1} \nonumber \\
    &\ \ \times \cdots  b_{k-1}^{p_{k-1}} (a_{k-1} - b_{k-1})^{1-p_{k-1}} \bigg) (a_k - b_k) \left(\prod_{p=k+1}^N b_{p}\right).
    \label{telescope2}
\end{align} 
whenever the additions and multiplications are well-defined on appropriate domain.  In the above, 

\begin{equation} 
\Omega := \{ p_1, \dots, p_{k-1} \in \{0,1\}\} \backslash \{p_1 = \dots = p_{k-1} = 1\}, 
\end{equation}
and we have used the convention that $\prod_{l=1}^0 a_l := I$ and $\prod_{p=N+1}^N b_p := I$, where $I$ denotes identity.
\end{lem}
\begin{proof}

\noindent (a) We prove by induction. For the base case of $N=2$, we have:
\begin{equation}
    a_1 a_2 - b_1 b_2 = a_1(a_2 - b_2) + (a_1 -b_1) b_2,
\end{equation}
and so \eqref{telescope} follows for $N=2$. Now, assume that \eqref{telescope} is true for $M>2$. Then,
\begin{align}
    \prod_{n=1}^{M+1} a_n - \prod_{m=1}^{M+1} b_m &= (a_1 \cdots a_M) (a_{M+1} - b_{M+1}) + \left(\prod_{n=1}^{M} a_n - \prod_{m=1}^{M} b_m \right) b_{M+1} \\
    &= (a_1 \cdots a_M) (a_{M+1} - b_{M+1}) + \sum_{k=1}^M \left(\prod_{l=1}^{k-1} a_l\right) (a_k - b_k) \left(\prod_{p=k+1}^M b_{p}\right) b_{M+1}\\
    &= \sum_{k=1}^{M+1} \left(\prod_{l=1}^{k-1} a_l\right) (a_k - b_k) \left(\prod_{p=k+1}^{M+1} b_{p}\right).
\end{align}
Therefore, the formula \eqref{telescope} holds for all $N \geq 2$. \\

\noindent (b) By part (a), we have:
\begin{align}
    \prod_{n=1}^N a_n - \prod_{m=1}^N b_m
    &= \sum_{k=1}^N \left(\prod_{l=1}^{k-1} a_l\right) (a_k - b_k) \left(\prod_{p=k+1}^N b_{p}\right) \\
    &= \sum_{k=1}^N \left(\prod_{l=1}^{k-1} (b_l + (a_l-b_l)) \right) (a_k - b_k) \left(\prod_{p=k+1}^N b_{p}\right). \label{a1}
\end{align} 

Note that
\begin{align}
    &\prod_{l=1}^{k-1} (b_l + (a_l-b_l)) \nonumber \\ 
    &=  \sum_{p_1 = 0}^1 \cdots \sum_{p_{k-1}=0}^1  b_1^{p_1} (a_1-b_1)^{1-p_1} \cdots  b_{k-1}^{p_{k-1}} (a_{k-1} - b_{k-1})^{1-p_{k-1}} \\
    &= \prod_{l=1}^{k-1} b_l + \sum_{p_1, \dots, p_{k-1} \in \Omega}   b_1^{p_1} (a_1-b_1)^{1-p_1} \cdots  b_{k-1}^{p_{k-1}} (a_{k-1} - b_{k-1})^{1-p_{k-1}}, \label{a2}
\end{align}
where $\Omega = \{ p_1, \dots, p_{k-1} \in \{0,1\}\} \backslash \{p_1 = \dots = p_{k-1} = 1\}$. 
Eq. \eqref{telescope} then follows from \eqref{a1} and \eqref{a2}. 
\end{proof}




\subsection{Proof of Proposition 3.1}
Our proof is adapted from and built on that  in \citep{chen2020mathematical}, which provides a rigorous justification of nonequilibrium FDTs for a wide class of diffusion processes and nonlinear input perturbations in the linear response regime (see also the more abstract approach in \citep{hairer2010simple,dembo2010markovian}). We are going to extend the techniques in \citep{chen2020mathematical} to study the fully nonlinear response regime in the context of our SRNNs.  

First, note that Assumption \ref{imp_ass} implies that the processes $\vecc{h}$ and $\overline{\vecc{h}}$ (for all $\gamma$) automatically satisfy the regular conditions (Definition 2.2) in \citep{chen2020mathematical}. Therefore, it follows from Proposition 2.5 in \citep{chen2020mathematical} that the weak solution of the SDE (5) exists up to time $T$ and is unique in law. In particular,   $\vecc{h}$ and $\overline{\vecc{h}}$  are nonexplosive up to time $T$. Moreover, if $n=r$ and $\vecc{\sigma}>0$, the strong solution of the SDE exists up to time $T$ and is pathwise unique. 

We collect some intermediate results that we will need  later. Recall that $\Delta \mathcal{L}_t := \mathcal{L}_t - \mathcal{L}^0$, for $t \in [0,T]$, where 
the infinitesimal generators $\mathcal{L}_t$ and $\mathcal{L}^0$ are defined in \eqref{7} and \eqref{8} in the main paper respectively. Also, we are using Einstein's summation convention. 

\begin{lem}  \label{chenl1}
For any $f \in C_b(\RR^n)$ and $0 \leq s \leq t \leq T$, the function $v(\vecc{h},s) = P_{s,t}^0 f(\vecc{h}) \in C_b(  \RR^n \times [0,t]) \cap C^{1,2}(\RR^n \times [0,t])$ is the unique bounded classical solution to the (parabolic) backward Kolmogorov equation (BKE):
\begin{align}
    \frac{\partial v}{\partial s} &= -\mathcal{L}^0 v, \ \ 0 \leq s < t, \label{bke1} \\
    v(\vecc{h},t) &= f(\vecc{h}). \label{bke2} 
\end{align}
If instead $f \in C_b^{\infty}(\RR^n)$, then the above statement holds with the BKE defined on $0 \leq s \leq t$ (i.e., with the endpoint $t$ included), in which case $v \in C^{1,2}(\RR^n \times [0,t])$. 
\end{lem}
\begin{proof}
The  statements are straightforward applications of Lemma 3.1 and Remark 3.2 in \citep{chen2020mathematical} to our setting, since Assumption \ref{imp_ass} ensures that the regular conditions there hold.  See the proof in \citep{chen2020mathematical} for details. The idea is that upon imposing some regular conditions, one can borrow the results from \citep{lorenzi2011optimal} to prove the existence and uniqueness of a bounded classical solution to the BKE \eqref{bke1}-\eqref{bke2}. That $v(\vecc{h},s) = P^0_{s,t} f(\vecc{h}) = \mathbb{E}[ f(\vecc{h}_t) | \vecc{h}_s = \vecc{h}]$ satisfies the BKE and the terminal condition \eqref{bke2} follows from an application of It\^o's formula and standard arguments of stochastic analysis \citep{karatzas1998brownian}. 
\end{proof}


\begin{lem} \label{lem_aux}
Let $\gamma \in C([0,T],\RR)$  and denote $P^\gamma_{s,t} := P_{s,t}$.  Then: \\
(a) For $f \in C_b^2(\RR^n)$ and $0 \leq s \leq t \leq T$, 
\begin{align}
    P^\gamma_{s,t} f(\vecc{h}) - P_{s,t}^0 f(\vecc{h}) &= \int_s^t du  \gamma(u)  U^i_u P_{s,u}^0    \frac{\partial}{\partial h^i}  P^\gamma_{u,t} f(\vecc{h}). \label{aa1}
\end{align}
(b)  For all $f \in C_b^{\infty}(\RR^n)$, $\phi \in C([0,T],\RR)$, and $0 \leq s \leq t \leq T$, 
\begin{align}
    \lim_{\epsilon \to 0}  \frac{1}{\epsilon} (P^{\epsilon \phi}_{0,t} f(\vecc{h}) - P_{0,t}^0 f(\vecc{h}) = \int_0^t \phi(s) P_{0,s}^0 \Delta \mathcal{L}_s P_{s,t}^0 f(\vecc{h}) ds,
\end{align}
where $\Delta \mathcal{L}_s \cdot = (\mathcal{L}_s - \mathcal{L}_s^0) \cdot =  U^i  \frac{\partial}{\partial h^i} \cdot$. 
\end{lem}
\begin{proof}
The following proof is based on and adapted from that of Lemma 3.5 and Lemma 3.7 in \citep{chen2020mathematical}. \\

\noindent (a) As noted earlier, due to Assumption \ref{imp_ass}, both $\vecc{h}$ and $\overline{\vecc{h}}$ satisfy the regular conditions (Definition 2.2) in \citep{chen2020mathematical}. Therefore, we can apply Lemma \ref{chenl1} and so their transition operators satisfy the BKE \eqref{bke1}. This implies that $u(\vecc{h},s) := P^\gamma_{s,t} f(\vecc{h}) - P^0_{s,t} f(\vecc{h})$ is the  bounded classical solution to:
\begin{align}
\frac{\partial u(\vecc{h},s)}{\partial s} &= -\mathcal{L}^0 u(\vecc{h},s) - \gamma(s) \Delta \mathcal{L}_s P^\gamma_{s,t} f(\vecc{h}), \ \ \ 0 \leq s < t, \label{b1} \\
u(\vecc{h},t) &= 0. \label{b2}
\end{align}

Applying It\^o's formula  gives
\begin{align}
    u(\overline{\vecc{h}}_s,s) &= \int_0^s \frac{\partial}{\partial r} u(\overline{\vecc{h}}_r, r) dr + \int_0^s \mathcal{L}^0 u(\overline{\vecc{h}}_r, r) dr + \int_0^s \nabla u(\overline{\vecc{h}}_r, r)^T \vecc{\sigma} d\vecc{W}_r \\
    &= - \int_0^s \gamma(r) \Delta \mathcal{L}_r P^\gamma_{r,t} f(\overline{\vecc{h}}_r)  dr + \int_0^s \nabla u(\overline{\vecc{h}}_r, r)^T \vecc{\sigma} d\vecc{W}_r. 
\end{align}
For $R>0$, let $B_R := \{ \vecc{h} \in \RR^n: |\vecc{h}| < R\}$. Define  $\tau_R := \inf \{ t \geq 0: \overline{\vecc{h}}_t \in  \partial B_R\}$, the hitting time of the sphere $\partial B_R$ by $\overline{\vecc{h}}$, and the explosion time $\tau = \lim_{R \to \infty} \tau_R$. Note that it follows from Assumption \ref{imp_ass} and an earlier note that $\tau > T$. 

Then, if $|\vecc{h}| < R$ and $s < r < t$, 
\begin{equation}
    u(\vecc{h}, s) = \mathbb{E}[ u(\overline{\vecc{h}}_{r \wedge \tau_R}, r \wedge \tau_R) | \overline{\vecc{h}}_s = \vecc{h}] + \mathbb{E}\left[ \int_s^r \gamma(u) \Delta \mathcal{L}_u P^\gamma_{u,t} f(\overline{\vecc{h}}_u)  1_{u \leq \tau_R} du \bigg| \overline{\vecc{h}}_s = \vecc{h} \right].  
\end{equation}
Since $\overline{\vecc{h}}$ is nonexplosive up to time $T$ and $u \in C_b(\RR^n \times [0,T])$, we have:
\begin{equation}
    \lim_{r \to t} \lim_{R \to \infty} \mathbb{E}[u(\overline{\vecc{h}}_{r \wedge \tau_R}, r \wedge \tau_R) | \overline{\vecc{h}}_s = \vecc{h} ] = \lim_{r \to t} \mathbb{E}[u(\overline{\vecc{h}}_r,r)| \overline{\vecc{h}}_s = \vecc{h}] = \mathbb{E}[u(\overline{\vecc{h}}_t,t) | \overline{\vecc{h}}_s = \vecc{h}] = 0.
\end{equation}
Note that $(\Delta \mathcal{L}_s P^\gamma_{s,t} f(\vecc{h}))_{s \in [0,t]} \in C_b(\RR^n \times [0,t])$ for any $f \in C_b^2(\RR^n)$ by Theorem 3.4 in \citep{chen2020mathematical} and recall that $\gamma \in C([0,t],\RR)$ for $t \in [0,T]$. Therefore, it follows from the above and an application of dominated convergence theorem that
\begin{equation}
    u(\vecc{h},s) = \int_s^t \gamma(u)  \mathbb{E}[\Delta \mathcal{L}_u P^\gamma_{u,t} f(\overline{\vecc{h}}_u)| \overline{\vecc{h}}_s = \vecc{h} ] du, \label{abc}
\end{equation}
from which \eqref{aa1} follows. \\

\noindent (b) By \eqref{abc} (with  $\gamma := \epsilon \phi$ and $s:=0$ there),  we have, for $\epsilon > 0$,  
\begin{align}
    \frac{1}{\epsilon}(P_{0,t}^{\epsilon \phi} f(\vecc{h}) - P^0_{0,t}f(\vecc{h})) &=  \int_0^t \phi(s) \mathbb{E}[ \Delta \mathcal{L}_s P^{\epsilon \phi}_{s,t} f(\overline{\vecc{h}}_s)| \overline{\vecc{h}}_0 = \vecc{h} ] ds.
\end{align}

We denote $g^\epsilon(\vecc{h},s) := \Delta \mathcal{L}_s P^{\epsilon \phi}_{s,t} f(\vecc{h})$. 
Then, it follows from Assumption \ref{imp_ass} that $\|g^\epsilon \|_\infty < \infty$, and
\begin{align}
    \sup_{s \in [0,t]} |\epsilon \phi(s)  g^\epsilon(\cdot,s) | &\leq  \epsilon \|\phi\|_\infty \|g^\epsilon\|_\infty \to 0, \label{b0}
\end{align}
as $\epsilon \to 0$. 

Now, let $t \in [0,T]$. For any $\alpha \in C_b^\infty(\RR^n \times [0,t])$, consider the following Cauchy problem:
\begin{align}
    \frac{\partial u(\vecc{h},s)}{\partial s} &= -\mathcal{L}^0 u(\vecc{h},s) - \alpha(\vecc{h},s), \ \ 0 \leq s \leq t, \label{b3} \\ 
    u(\vecc{h},t) &= 0.
\end{align}
By Theorem 2.7 in \citep{lorenzi2011optimal}, the above equation has a unique bounded classical solution. Moreover,  there exists a constant $C>0$ such that 
\begin{equation} \label{b4}
    \|u\|_\infty \leq  C  \|\alpha\|_\infty. 
\end{equation}
Therefore, \eqref{b1}-\eqref{b2} together with  \eqref{b0}-\eqref{b4} give: 
\begin{equation}
    \sup_{s \in [0,t]} |P_{s,t}^{\epsilon \phi} f - P_{s,t} f| \to 0,
\end{equation}
as $\epsilon \to 0$. 
Thus, we have $g^\epsilon(\vecc{h},s) \to \Delta \mathcal{L}_s P_{s,t}^0 f(\vecc{h})$ as $\epsilon \to 0$. 

Finally, using all these and the dominated convergence theorem, we have:
\begin{align}
    \lim_{\epsilon \to 0} \frac{1}{\epsilon}(P_{0,t}^{\epsilon \phi} f(\vecc{h}) - P^0_{0,t}f(\vecc{h})) &=  \int_0^t \phi(s)  \mathbb{E}\left[  \lim_{\epsilon \to 0} g^\epsilon(\overline{\vecc{h}}_s, s) \bigg| \overline{\vecc{h}}_0 = \vecc{h} \right] ds \\
    &= \int_0^t \phi(s)  \mathbb{E}\left[ \Delta \mathcal{L}_s P_{s,t}^0(\overline{\vecc{h}}_s) \bigg| \overline{\vecc{h}}_0 = \vecc{h} \right] ds.
\end{align}

\end{proof}

We now prove Proposition 3.1.\\

\begin{proof}(Proof of Proposition 3.1 (a)) 
Recall that it follows from our assumptions that  $f(\vecc{h}_t) \in C_b^{\infty}(\RR^n)$ for all $t \in [0,T]$. 

We proceed by induction. For the base case of $n=1$, we have, for $0 \leq t \leq T$, 
\begin{equation}
\mathbb{E} f(\vecc{h}_t) = \int P^0_{0,t}f(\vecc{h}) \rho_{init}(\vecc{h}) d\vecc{h} = \mathbb{E} P^0_{0,t} f(\vecc{h}_0). 
\end{equation}
Then, for $u_0 \in C([0,t],\RR)$  and any $\phi  \in C_c^\infty(0,t)$:
\begin{align}
    \int_0^t \frac{\delta F_s}{\delta u}\bigg|_{u=u_0} \phi(s) ds &= DF_t[u_0] \cdot \phi \\
    &= \lim_{\epsilon \to 0} \frac{1}{\epsilon} (F_t[u_0 + \epsilon \phi] - F_t[u_0]) \\
    &= \lim_{\epsilon \to 0} \frac{1}{\epsilon} (\mathbb{E} f(\vecc{h}_t^{u_0 + \epsilon \phi}) - \mathbb{E} f(\vecc{h}^{u_0}_t)) \\ 
    &= \lim_{\epsilon \to 0} \mathbb{E} \left[ \frac{1}{\epsilon} (P_{0,t}^{u_0 + \epsilon \phi} f(\vecc{h}_0) - P^{u_0}_{0,t} f(\vecc{h}_0)) \right] \\
    &= \int_0^t ds \phi(s) \mathbb{E} P^{u_0}_{0,s} \Delta L_s P^{u_0}_{s,t} f(\vecc{h}_0), 
\end{align}
where the last equality follows from Lemma \ref{lem_aux} (b).  Therefore, the result for the base case follows upon setting $u_0 = 0$.

Now assume that 
\begin{align} 
D^{n-1} F_t[u_0] \cdot (\phi)^{n-1} &= (n-1)! \int_{[0,t]^{n-1}} ds_1 \cdots ds_{n-1} \phi(s_1) \cdots \phi(s_{n-1}) \nonumber \\
&\hspace{0.5cm} \times \mathbb{E} P^{u_0}_{0,s_{n-1}} \Delta \mathcal{L}_{s_{n-1}} P^{u_0}_{s_{n-1},s_{n-2}} \Delta \mathcal{L}_{s_{n-2}} \cdots P^{u_0}_{s_1,t} f(\vecc{h}_0) 
\end{align}
holds for any $\phi \in C_c^\infty(0,t)$,  for $n>1$. Then, for any $\phi \in C_c^\infty(0,t)$:
\begin{align} 
&D^{n} F_t[0] \cdot (\phi)^n\\
&= \lim_{\epsilon \to 0}  \frac{(D^{n-1} F_t[\epsilon \phi] - D^{n-1}F_t[0])\cdot (\phi)^{n-1}}{\epsilon}  \\
&= (n-1)! \lim_{\epsilon \to 0} \frac{1}{\epsilon} \bigg( \int_{[0,t]^{n-1}} ds_1 \cdots ds_{n-1} \phi(s_1) \cdots \phi(s_{n-1}) \big(\mathbb{E} P^{\epsilon \phi}_{0,s_{n-1}} \Delta \mathcal{L}_{s_{n-1}} \nonumber \\
&\hspace{0.5cm} P^{\epsilon \phi}_{s_{n-1},s_{n-2}} \Delta \mathcal{L}_{s_{n-2}} \cdots P^{\epsilon \phi}_{s_1,t} f(\vecc{h}_0)  - \mathbb{E} P^{0}_{0,s_{n-1}} \Delta \mathcal{L}_{s_{n-1}} P^{0}_{s_{n-1},s_{n-2}} \Delta \mathcal{L}_{s_{n-2}} \cdots P^{0}_{s_1,t} f(\vecc{h}_0) \big) \bigg).
\end{align}

Note that by use of Lemma \ref{lem_teles}(b) and that the limit of products of two or more terms of the form $P_{s,s'}^{\epsilon \phi} - P^0_{s, s'}$ (with $s \leq s'$), when multiplied by $(1/\epsilon)$, vanishes as $\epsilon \to 0$, we have
\begin{align}
&\frac{1}{\epsilon} \big( \mathbb{E} P^{\epsilon \phi}_{0,s_{n-1}} \Delta \mathcal{L}_{s_{n-1}} P^{\epsilon \phi}_{s_{n-1},s_{n-2}} \Delta \mathcal{L}_{s_{n-2}} \cdots P^{\epsilon \phi}_{s_1,t} f(\vecc{h}_0) \nonumber \\ 
&\hspace{0.5cm} - \mathbb{E} P^{0}_{0,s_{n-1}} \Delta \mathcal{L}_{s_{n-1}} P^{0}_{s_{n-1},s_{n-2}} \Delta \mathcal{L}_{s_{n-2}} \cdots P^{0}_{s_1,t} f(\vecc{h}_0) \big) \\
&= \frac{1}{\epsilon} \mathbb{E} \sum_{k=1}^{n-1} \left( \prod_{l=1}^{k-1} \Delta \mathcal{L}_{s_l} P^0_{s_l, s_{l-1}} \right) (\Delta \mathcal{L}_{s_k} P_{s_k, s_{k-1}}^{\epsilon \phi} - \Delta \mathcal{L}_{s_k} P^0_{s_k, s_{k-1}}) \left( \prod_{p=k+1}^n \Delta \mathcal{L}_{s_p} P^0_{s_p, s_{p-1}} \right) f(\vecc{h}_0) \nonumber \\ 
&\hspace{0.5cm} + e(\epsilon),
\end{align}
where $e(\epsilon) = o(\epsilon)$ as $\epsilon \to 0$, and we have set $s_0 := t$, $s_n := 0$ and $\Delta \mathcal{L}_{s_n} := 1$. 

Moreover, 
\begin{align}
&\lim_{\epsilon \to 0} \frac{1}{\epsilon} \big( \mathbb{E} P^{\epsilon \phi}_{0,s_{n-1}} \Delta \mathcal{L}_{s_{n-1}} P^{\epsilon \phi}_{s_{n-1},s_{n-2}} \Delta \mathcal{L}_{s_{n-2}} \cdots P^{\epsilon \phi}_{s_1,t} f(\vecc{h}_0) \nonumber \\
&\hspace{0.5cm} - \mathbb{E} P^{0}_{0,s_{n-1}} \Delta \mathcal{L}_{s_{n-1}} P^{0}_{s_{n-1},s_{n-2}} \Delta \mathcal{L}_{s_{n-2}} \cdots P^{0}_{s_1,t} f(\vecc{h}_0) \big) \\
&=  \mathbb{E} \sum_{k=1}^{n} \left( \prod_{l=1}^{k-1} \Delta \mathcal{L}_{s_l} P^0_{s_l, s_{l-1}} \right) \Delta \mathcal{L}_{s_k} \frac{1}{\epsilon} \lim_{\epsilon \to 0} (P_{s_k, s_{k-1}}^{\epsilon \phi} -  P^0_{s_k, s_{k-1}}) \left( \prod_{p=k+1}^n \Delta \mathcal{L}_{s_p} P^0_{s_p, s_{p-1}} \right) f(\vecc{h}_0) \\
&=  \mathbb{E} \sum_{k=1}^{n} \left( \prod_{l=1}^{k-1} \Delta \mathcal{L}_{s_l} P^0_{s_l, s_{l-1}} \right) \Delta \mathcal{L}_{s_k} 
\int_0^t ds \phi(s) P_{s_k,s}^0 \Delta \mathcal{L}_s P_{s,s_{k-1}}^0 
\left( \prod_{p=k+1}^n \Delta \mathcal{L}_{s_p} P^0_{s_p, s_{p-1}} \right) f(\vecc{h}_0),
\end{align}
where we have applied Lemma \ref{lem_aux}(b) in the last line. 

Hence, using the above expression and symmetry of the mapping associated to derivative, we have
\begin{align} 
D^{n} F_t[0] \cdot (\phi)^n
&= n (n-1)! \int_{[0,t]^{n}}  ds_1 \cdots ds_{n-1} ds \phi(s)  \phi(s_1) \cdots \phi(s_{n-1}) \nonumber \\
&\hspace{0.5cm} \times \mathbb{E} P^0_{0,s} \Delta \mathcal{L}_{s}  P^0_{s,s_{n-1}} \Delta \mathcal{L}_{s_{n-1}} P^0_{s_{n-1},s_{n-2}} \Delta \mathcal{L}_{s_{n-2}} \cdots P^0_{s_1,t} f(\vecc{h}_0). \label{crucial}
\end{align}
Therefore, (a) holds for all $n \geq 2$. 
\end{proof}

\begin{rmk}
Note that here it is crucial to have (b)-(c) in Assumption \ref{imp_ass} to ensure that all derivatives of the form \eqref{crucial} are bounded and Lipschitz continuous. It may be possible to relax the assumptions on $f$ at an increased cost of technicality but we choose not to pursue this direction.  Had we been only interested in the linear response regime (i.e., $n=1$ case), then one can indeed relax the assumption on $f$ substantially (see \citep{chen2020mathematical}).  
\end{rmk}

\begin{proof}(Proof of Proposition 3.1 (b))
We proceed by an induction argument. We use the notation $(f,g) := \int_{\RR^n} f(\vecc{h}) g(\vecc{h}) d\vecc{h}$ for $f,g \in C_b(\RR^n)$, in the following. Let $p_t$ denote the probability density of $\overline{\vecc{h}}_t$, $t \in [0,T]$, and recall that $p_0 = \rho_{init}$.  

For the base case of $n=1$, first note that, using the properties of expectation,
\begin{equation}
    \mathbb{E}[P^0_{0,s_1} \Delta \mathcal{L}_{s_1} P^0_{s_1,t} f(\vecc{h}_0)] = \mathbb{E}[ \mathbb{E}[ \Delta \mathcal{L}_{s_1} P^0_{s_1,t} f(\vecc{h}_{s_1})| \vecc{h}_0 = \vecc{h}_0 ]] = \mathbb{E}[ \Delta \mathcal{L}_{s_1} P^0_{s_1,t} f(\vecc{h}_{s_1})],   
\end{equation}
for any $s_1 < t$. Therefore, by part (a) and applying integration by parts: 
\begin{align}
    R_f^{(1)}(t,s_1) &=  \mathbb{E}[ \Delta \mathcal{L}_{s_1} P^0_{s_1,t} f(\vecc{h}_{s_1})] \\
    &= \int_{\RR^n}  \Delta \mathcal{L}_{s_1} P^0_{s_1,t} f(\vecc{h}) p_{s_1}(\vecc{h}) d\vecc{h} \\
    &= \left( \Delta \mathcal{L}_{s_1} P^0_{s_1,t} f,  p_{s_1}  \right)\\
    &= \left( P^0_{s_1,t} f, \Delta \mathcal{A}_{s_1}  p_{s_1}   \right) \\
    &=  \left( f, ((P^0_{s_1,t})^* \Delta \mathcal{A}_{s_1}  p_{s_1} )  \right)\\
    &= \left( f, [((P^0_{s_1,t})^* \Delta \mathcal{A}_{s_1}  p_{s_1} ) \rho_{init}^{-1} ]  \rho_{init}  \right) \\
    &= \int_{\RR^n}  f(\vecc{h}) v_{t,s_1}^{(1)}(\vecc{h}_{s_1}) p_0(\vecc{h}) d\vecc{h} 
\end{align}
where the second last line is well-defined since $\rho_{init}(\vecc{h}) = p_0(\vecc{h}) > 0$. 
Therefore, 
\begin{align}
    R_f^{(1)}(t,s_1) &= \mathbb{E} f(\vecc{h}_0)v_{t,s_1}^{(1)}(\vecc{h}_{s_1}). 
\end{align}

Now, assume that \eqref{33}-\eqref{con_obs} in the main paper holds for $n=k$. Note that
\begin{align}
    &\mathbb{E}[P^0_{0,s_n} \Delta \mathcal{L}_{s_n} P^0_{s_n,s_{n-1}} \Delta \mathcal{L}_{s_{n-1}} \cdots P^0_{s_1,t} f(\vecc{h}_0)] \\ 
    &= \mathbb{E}[ \mathbb{E}[ \Delta \mathcal{L}_{s_n} P^0_{s_n,s_{n-1}} \Delta \mathcal{L}_{s_{n-1}} \cdots P^0_{s_1,t} f(\vecc{h}_{s_n}) | \vecc{h}_0 = \vecc{h}_0 ]]\\
    &= \mathbb{E}[ \Delta \mathcal{L}_{s_n} P^0_{s_n,s_{n-1}} \Delta \mathcal{L}_{s_{n-1}} \cdots P^0_{s_1,t} f(\vecc{h}_{s_n})]
\end{align}
for any $n$.

Then, by part (a),
\begin{align}
    R_f^{(k+1)}(t,s_1,\dots,s_{k+1}) &= \mathbb{E}[ \Delta \mathcal{L}_{s_{k+1}} P^0_{s_{k+1},s_{k}} \Delta \mathcal{L}_{s_{k}} \cdots P^0_{s_1,t} f(\vecc{h}_{s_{k+1}})] \\
    &= \int_{\RR^n}  \Delta \mathcal{L}_{s_{k+1}} P^0_{s_{k+1},s_k} \Delta \mathcal{L}_{s_k} \cdots P^0_{s_1, t}  f(\vecc{h}) p_{s_{k+1}}(\vecc{h}) d\vecc{h} \\
    &= \left( \Delta \mathcal{L}_{s_{k+1}} P^0_{s_{k+1},s_k} \Delta \mathcal{L}_{s_k} \cdots P^0_{s_1, t}  f,  p_{s_{k+1}}  \right)\\
    &= \left(   f,  (P^0_{s_1, t})^* \Delta \mathcal{A}_{s_1}  (P^0_{s_2,s_1})^* \cdots \Delta \mathcal{A}_{s_k} (P^0_{s_{k+1},s_k})^* \Delta \mathcal{A}_{s_{k+1}} p_{s_{k+1}}  \right)\\
    &= \int_{\RR^n}  f(\vecc{h}) v_{t,s_1,\dots,s_{k+1}}^{(k+1)}(\vecc{h}_{s_1}, \dots, \vecc{h}_{s_{k+1}}) p_0(\vecc{h}) d\vecc{h}, \label{85} 
\end{align}
where $v_{t,s_1,\dots,s_{k+1}}^{(k+1)}(\vecc{h}_{s_1}, \dots, \vecc{h}_{s_{k+1}})$ is given by \eqref{con_obs} in the main paper.  Note that we have applied integration by parts multiple times to get the second last line above. Therefore, 
\begin{align}
    R_f^{(k+1)}(t,s_1,\dots,s_{k+1}) &= \mathbb{E} f(\vecc{h}_0)v_{t,s_1,\dots, s_{k+1}}^{(k+1)}(\vecc{h}_{s_1},\dots, \vecc{h}_{s_{k+1}}). 
\end{align}
The proof is done.


\end{proof}

\subsection{Proof of Corollary 3.1}
\begin{proof}(Proof of Corollary 3.1)
It follows from \eqref{85} that  for $n \geq 1$, $f \in C_c^\infty$,
\begin{align}
    &\int_{\RR^n}  f(\vecc{h}) (v_{t,s_1,\dots,s_{n}}^{(n)}(\vecc{h}_{s_1}, \dots, \vecc{h}_{s_{n}}) - \tilde{v}_{t,s_1,\dots,s_{n}}^{(n)}(\vecc{h}_{s_1}, \dots, \vecc{h}_{s_{n}})) \rho_{init}(\vecc{h}) d\vecc{h} = 0.
\end{align}
 Since $f$ is arbitrary and $\rho_{init} > 0$, the result follows.
\end{proof}

\subsection{Proof of Theorem 3.1}
\begin{proof}(Proof of Theorem 3.1)
Recall that $\gamma := (\gamma(s) := |\vecc{C} \vecc{u}_s|)_{s \in [0,t]}  \in C([0,t],\RR)$  for $t \in [0,T]$ by Lemma \ref{B2}. 

Associated with $\mathbb{E} f(\vecc{h}_t)$ is the mapping $F_t: C([0,t],\RR) \to \RR$, $\gamma \to \mathbb{E}f(\vecc{h}_t)$, where $\vecc{h}$ is the hidden state of the SRNN. Since by our assumptions the $D^n F_t[0] \cdot (\gamma, \dots, \gamma)$ are well-defined  for $n \in \ZZ_+$, the mapping $F_t$ admits an absolutely convergent Taylor series at the point $0$ for sufficiently small $\gamma$:
\begin{equation}
F_t[\gamma] =   \sum_{n=1}^\infty \frac{1}{n!} D^n F_t[0] \cdot (\gamma)^n,
\end{equation}
where $(\gamma)^n := (\gamma, \dots, \gamma)$ ($n$ times).  Moreover, the derivatives are bounded and Lipschitz continuous, therefore integrable on compact sets. They can be identified with the response kernels $R_f^{(n)}(t,\cdot)$  given in Proposition 3.1 in the main paper. The resulting series is a Volterra series in the sense of Definition \ref{def_vol}. The uniqueness follows from the symmetry of the  derivative mappings.
\end{proof}

\subsection{Proof of Theorem 3.2} \label{sect_b4}
We start with the proof and then give a few remarks.\\

\begin{proof}(Proof of Theorem 3.2)
By assumption,  an eigenfunction expansion of the operator $\mathcal{A}^0$ exists and is well-defined. We consider  the eigenvalue-eigenfunction pairs $(-\lambda_m, \phi_m)_{m \in \ZZ_+}$, i.e., $\mathcal{A}^0 \phi_m = \lambda_m \phi_m$ (for $m \in \ZZ_+$),  where the $\lambda_m \in \CC$ and the $\phi_m \in L^2(\rho_{init})$ are orthonormal eigenfunctions that span  $L^2(\rho_{init})$ (see also Remark \ref{rmk_symmetric}). 

In this case, we have  $e^{\mathcal{A}^0 t} \phi_m = e^{-\lambda_m t} \phi_m$, which implies that for any $f \in L^2(\rho_{init})$ we have $e^{\mathcal{A}^0 t} f(\vecc{h}) = \sum_n \alpha_n(\vecc{h}) e^{-\lambda_n t}$, where $\alpha_n(\vecc{h}) = \langle \phi_n, f \rangle \phi_n(\vecc{h})$. 

We first derive formula \eqref{ultimate_main}. 
Applying the above representation in \eqref{24}-\eqref{26} in the main paper, we arrive at the following formula for the response kernels (recall that we are using Einstein's summation notation for repeated indices):

\begin{align} \label{separate}
&\mathcal{K}^{\vecc{k}^{(n)}}(s_0,s_1,\dots, s_{n}) =  e^{-\lambda_{m_n} s_n} e^{-\lambda_{l_1} (s_0-s_{1})} \cdots e^{-\lambda_{l_n} (s_{n-1}-s_{n})}  Z^{\vecc{k}^{(n)}}_{l_1,\dots, l_n,m_n},
\end{align}
for $n \in \ZZ_+$, where  
\begin{equation}
Z^{\vecc{k}^{(n)}}_{l_1,\dots, l_n,m_n}= (-1)^n \int d\vecc{h} f(\vecc{h}) \alpha_{l_1}(\vecc{h}) \frac{\partial}{\partial h^{k_1}}\left[ \cdots \alpha_{l_n}(\vecc{h}) \frac{\partial}{\partial h^{k_n}}\left[\alpha_{m_n}(\vecc{h}) \rho_{init}(\vecc{h}) \right]  \right],
\end{equation}
which can be written as average of a functional with respect to $\rho_{init}$. In the above, $\alpha_{m} = \langle \phi_{m}, \rho_{init}\rangle \phi_{m}$,
and 
\begin{align}
\alpha_{l_1} &= \left\langle \phi_{l_1}, \frac{\partial}{\partial h^{k_1}} (\alpha_{m_n}(\vecc{h}) \right\rangle, \\
\alpha_{l_n} &= \left\langle \phi_{l_1}, \frac{\partial}{\partial h^{k_1}} ( \phi_{l_{1}}(\vecc{h}) \cdots \frac{\partial}{\partial h^{k_{n-1}}} (\phi_{l_{n-1}}(\vecc{h}) \frac{\partial}{\partial h^{k_n}}(  \alpha_{m_n}(\vecc{h})))) \right\rangle,
\end{align}
for $n=2,3,\dots$.

Plugging in the above expressions into \eqref{volt} in the main paper, we  cast $\mathbb{E} f(\vecc{h}_t)$ into  a series of  generalized convolution integrals with exponential weights:
\begin{align}
\mathbb{E} f(\vecc{h}_t) &=  \sum_{n=1}^{\infty}   \epsilon^n  Z^{ \vecc{k}^{(n)}}_{l_1,\dots, l_n,m_n} e^{-\lambda_{l_1} t} \\ 
&\ \times \int_0^t ds_1  \tilde{U}_{s_1}^{k_1} \cdots \int_0^{s_{n-1}} ds_n \tilde{U}_{s_n}^{k_n} e^{-\lambda_{m_n} s_n} e^{-\lambda_{l_1}(s_0-s_1)} \cdots e^{-\lambda_{l_n(s_{n-1}-s_n)}}\\
&= \sum_{n=1}^{\infty}   \epsilon^n  Z^{ \vecc{k}^{(n)}}_{l_1,\dots, l_n,m_n} e^{-\lambda_{l_1} t} \int_0^t ds_1 \tilde{U}_{s_1}^{k_1} e^{-(\lambda_{l_2}-\lambda_{l_1})s_1} \cdots \int_0^{s_{n-1}}ds_n \tilde{U}_{s_n}^{k_n} e^{-(\lambda_{m_n}-\lambda_{l_n})s_n} \label{151}
\end{align}
(with the $\lambda_{m_n}$   equal zero  in the case of stationary invariant distribution).  Note that the expression above is obtained after performing  interchanges between integrals and summations, which are justified by Fubini's theorem.


To isolate the unperturbed part of SRNN from $\tilde{U}$ completely, we expand the exponentials in Eq. \eqref{151} in power series to obtain:
\begin{align}
&\mathbb{E} f(\vecc{h}_t) \nonumber \\ 
&= \sum_{n=1}^{\infty}  \epsilon^n  Z^{\vecc{k}^{(n)}}_{l_1,\dots, l_n, m_n} (-1)^{p_0+\dots+p_{n}} \frac{(\lambda_{l_1})^{p_0}}{p_0!} \frac{(\lambda_{l_2}-\lambda_{l_1})^{p_1}}{p_1!}   \cdots \frac{(\lambda_{m_{n}}-\lambda_{l_{n}})^{p_n}}{p_n!} \nonumber \\
&\ \ \ \ \ \times   \left( t^{p_0} \int_0^{t} ds_1    (s_1)^{p_1}  \tilde{U}_{s_1}^{k_1}   \cdots \int_0^{s_{n-1}} ds_n (s_n)^{p_n} \tilde{U}_{s_n}^{k_n}\right) \\
&=:  \sum_{n=1}^{\infty}  \epsilon^n Q_{\vecc{p}^{(n)}}^{\vecc{k}^{(n)}} \left( t^{p_0}  \int_0^{t} ds_1    s_1^{p_1}  \tilde{U}_{s_1}^{k_1}    \cdots \int_0^{s_{n-1}} ds_{n}  s_{n}^{p_{n}}  \tilde{U}_{s_{n}}^{k_{n}} \right). \label{ultimate}
\end{align}
This is the formula \eqref{ultimate_main} in the main paper. The series representation in \eqref{memless} in the main paper and its convergence then follows from the above result and Assumption \ref{imp_ass}. In particular, the constant coefficients $a_{p_0, \dots, p_n, l_1, \dots, l_n}$  in \eqref{memless} are given by:
\begin{equation}\label{exp_exp}
    a_{p_0, \dots, p_n, l_1, \dots, l_n} = Q_{\vecc{p}^{(n)}}^{\vecc{k}^{(n)}} C^{k_1 l_1} \cdots C^{k_n l_n},
\end{equation}
where the $Q_{\vecc{p}^{(n)}}^{\vecc{k}^{(n)}}$ are defined in \eqref{ultimate} (recall that we are using Einstein's summation notation for repeated indices).
\end{proof}

\begin{rmk} \label{rmk_symmetric} 
If the operator $\mathcal{A}^0$ is symmetric, then such eigenfunction expansion exists and is unique, the $\lambda_m \in \RR$, and, moreover, the real parts of  $\lambda_m$ are positive for exponentially stable SRNNs \citep{pavliotis2014stochastic}. Working with the eigenfunction basis  allows us to ``linearize'' the SRNN dynamics, thereby identifying the dominant directions and time scales of the unperturbed part of the SRNNs.  If, in addition, the unperturbed SRNN is stationary  and ergodic, one can, using Birkhoff's ergodic theorem, estimate the spatial average in the response functions with time averages.
\end{rmk}

\begin{rmk} \label{rmk_visual}
Eq. \eqref{separate} disentangles the time-dependent component from the static component described by the $Z^{\vecc{k}^{(n)}}_{l_1,\dots, l_n,m_n}$.  The time-dependent component is solely determined by the eigenvalues $\lambda_i$'s, which give us the set of memory time scales on which the response kernels evolve.  The static component is dependent on the eigenfunctions $\phi_n$'s, as well as on the initial distribution  $\rho_{init}$ of the hidden state, and the activation function $f$.  The eigenvalues and eigenfunctions are determined by the choice of activation function (and their parameters) and the noise in the hidden states of SRNN. In practice, the multiple infinite series above can be approximated by a finite one by keeping only the dominant contributions, thereby giving us a feasible way to visualize and control the internal dynamics of SRNN by manipulating the response functions.
\end{rmk}


\subsection{Proof of Proposition 3.3} \label{appendix_bell}

The computation involved in deriving Eq. (36) in the main paper essentially comes from the following  combinatorial result.

\begin{lem} \label{bell}
Consider the formal power series 
$a(x) = \sum_{i=1}^\infty a_i x^i$ and $b(x) = \sum_{i=1}^\infty b_i x^i$ in one symbol (indeterminate) with coefficients in a field. Their composition $a(b(x))$ is again a formal power series, given by:
\begin{equation}
    a(b(x)) = \sum_{n=1}^\infty c_n,
\end{equation}
with 
\begin{equation}
    c_n =  \sum_{\mathcal{C}_n} a_k b_{i_1} \cdots b_{i_k},
\end{equation}
where $\mathcal{C}_n = \{(i_1, \cdots, i_k) : 1 \leq k \leq n, \ i_1 + \cdots + i_k = n \}$. 
If $a_i := \alpha_i/i!$ and $b_i := \beta_i/i!$ in the above, then the $c_n$ can be expressed in terms of the exponential Bell polynomials $B_{n,k}$ \citep{bell1927partition}:
\begin{equation}
    c_n =  \frac{1}{n!} \sum_{k=1}^n \alpha_k B_{n,k}(\beta_1, \dots, \beta_{n-k+1}),
\end{equation}
where 
\begin{align}
    &B_{n,k}(\beta_1, \dots, \beta_{n-k+1}) \nonumber \\
    &= \sum_{c_1,c_2,\dots, c_{n-k+1} \in \ZZ^+} \frac{n!}{c_1! c_2! \cdots c_{n-k+1}!} \left(\frac{\beta_1}{1!}\right)^{c_1} \left(\frac{\beta_2}{2!}\right)^{c_2} \cdots \left(\frac{\beta_{n-k+1}}{(n-k+1)!}\right)^{c_{n-k+1}}   
\end{align}
are homogeneous polynomials of degree $k$
such that $c_1+2c_2+3c_3 +  \dots + (n-k+1) c_{n-k+1} = n$ and $c_1 + c_2 + c_3 + \cdots + c_{n-k+1} = k$. 
\end{lem}
\begin{proof}
See, for instance, Theorem A in Section 3.4 in \citep{comtet2012advanced}  (see also Theorem 5.1.4 in \citep{stanley1999enumerative}). 
\end{proof}

For a statement concerning the radius of convergence of composition of formal series, see Proposition 5.1 in Section 1.2.5 in \citep{cartan1995elementary}. See also Theorem C in Section 3.4 in \citep{comtet2012advanced} for relation between the above result with the $n$th order derivative of $a(b(x))$ when $a$, $b$  are smooth functions of a real variable. \\
 
\begin{proof}(Proof of Proposition 3.3)
Since Volterra series are  power series (in the input variable) of operators from $L^\infty$ to $L^\infty$ (see the remarks after Definition \ref{def_vol}), the idea is to apply Lemma \ref{bell}, from which (36) in the main paper follows.

Denote $\|\cdot \| := \| \cdot \|_\infty$ in the following.  Note that for any finite positive $r$:
\begin{equation}
    \|R_{fg}^{(r)}\| \leq \sum_{k=1}^r \sum_{\mathcal{C}_r} \|R_f^{(k)}\| \|R_g^{(i_1)}\| \cdots \|R_g^{(i_k)}\| < \infty \label{105}
\end{equation}
since the response functions in the finite series above are bounded by our assumptions. This justify any interchange of integrals and summations in the truncated Volterra series that involves during computation and therefore we can proceed and apply Lemma \ref{bell}.

For the case of (infinite) Volterra series, to ensure  absolute convergence of the series we suppose\footnote{If the resulting formal series is not convergent, then one needs to treat it as an asymptotic expansion.} that $a_F(a_G(\|u\|))<\infty$, where $a_F(x) = \sum_{n=1}^\infty \|R_f^{(n)}\|x^n$ and $a_G(x) = \sum_{n=1}^\infty \|R_g^{(n)}\|x^n$. Then, $G_t[u]$ is well-defined as a Volterra series and $\|G_t[u]\| \leq a_G(\|u\|)$. Also, $F_t[G_t[u]]$ is well-defined as a Volterra series by our assumption. In particular, it follows from \eqref{105}  that $a_{F \circ G}(\|u\|) \leq a_F(a_G(\|u\|))$, where $a_{F \circ  G}(x) = \sum_{n=1}^\infty \|R_{fg}^{(n)}\| x^n$. Thus, if $\rho_F$, $\rho_G$, $\rho_{F \circ G}$ denote the radius of convergence of $F_t$, $G_t$ and $(F \circ G)_t$ respectively, then $\rho_{F \cdot G} \geq \min(\rho_{G}, a_G^{-1}(\rho_F))$.   The last statement in the proposition follows from the above results.


\end{proof}

\subsection{Proof of Theorem 3.4}
\begin{proof}(Proof of Theorem 3.4 -- primal formulation) 
Let $p \in \ZZ_+$. First, note that by Lemma \ref{bv} any increasing function on $[0,T]$ has finite one-variation on $[0,T]$. Since  components of $\vecc{\psi}^{(p)}$, being real-valued monomials on $[0,T]$, are increasing functions on $[0,T]$ and therefore they are  of bounded variation on $[0,T]$. Since product of functions of bounded variation on $[0,T]$ is also of bounded variation  on $[0,T]$, all entries of the matrix-valued paths $\vecc{X}^{(p)} = \vecc{u} \otimes \vecc{\psi}^{(p)}$ are of bounded variation on $[0,T]$, and thus the signature of the $\vecc{X}^{(p)}$ is well-defined.  The result then essentially follows from Theorem 3.2 and Definition 3.2 in the main paper.
\end{proof}

\subsection{Proof of Proposition 3.2} \label{B9}

First, we recall the definition of symmetric Fock space, which has played an important role in stochastic processes \citep{guichardet2006symmetric}, quantum probability \citep{parthasarathy2012introduction}, quantum field theory \citep{goldfarb2013many}, and systems identification \citep{zyla1983nonlinear}. 

In the following, all Hilbert spaces are real.

\begin{defn}(Symmetric Fock space) Let $\mathcal{H}$ be a Hilbert space and $n \geq 2$ be any integer. For $h_1, \dots, h_n \in \mathcal{H}$, we define the symmetric tensor product:
\begin{equation}
h_1 \circ \cdots \circ h_n = \frac{1}{n!} \sum_{\sigma \in S_n} h_{\sigma(1)} \otimes \cdots \otimes h_{\sigma(n)},
\end{equation}
where $S_n$ denotes the group of all permutations of the set $\{1,2,\dots,n\}$. By $n$-fold symmetric tensor product of $\mathcal{H}$, denoted $\mathcal{H}^{\circ n}$, we mean the closed subspace of $\mathcal{H}^{\otimes n}$ generated by the $h_1, \dots, h_n$.  It is equipped with the inner product:
\begin{equation}
\langle u_1 \circ \cdots \circ u_n, v_1 \circ \cdots \circ v_n \rangle_{\mathcal{H}^{\circ n}} = Per(\langle u_i, v_j \rangle)_{ij},
\end{equation}
where Per denotes the permanent (i.e., the determinant without the minus sign) of the matrix. Note that $\|u_1 \circ \cdots \circ u_n\|_{\mathcal{H}^{\circ n}}^2 = n! \|u_1 \circ \cdots \circ u_n \|^2_{\mathcal{H}^{\otimes n}}$.

The symmetric Fock space over $\mathcal{H}$ is defined as $T_s((\mathcal{H})) = \bigoplus_{n=0}^\infty \mathcal{H}^{\circ n}$, with $\mathcal{H}^{\circ 0} := \RR$, $\mathcal{H}^{\circ 1} := \mathcal{H}$. It is equipped with the inner product 
\begin{equation}
\langle H, K \rangle_{T_s((\mathcal{H}))} = \sum_{n=0}^\infty n! \langle h_n, k_n \rangle_{\mathcal{H}^{\otimes n}},
\end{equation}
for elements $H := (h_m)_{m \in \NN}$, $K := (k_m)_{m \in \NN}$ in $T_s((\mathcal{H}))$, i.e., $h_m, k_m \in \mathcal{H}^{\circ m}$ for $m \in \NN$, and $\|H\|_{T_s((\mathcal{H}))}^2, \|K\|_{T_s((\mathcal{H}))}^2  < \infty$.
\end{defn}
It can be shown that the symmetric Fock space can be obtained from a free Fock space by applying appropriate projection. One advantage of working with the symmetric Fock space instead of the free one is that the symmetric space enjoys a functorial property that the free space does not have (see \citep{parthasarathy2012introduction}). Moreover, it satisfies the following  property that we will need later.

\begin{lem}(Exponential property) \label{exp_prop}
Let $h \in \mathcal{H}$ and define the element (c.f. the so-called coherent state in \citep{parthasarathy2012introduction}) $e(h) := \bigoplus_{n=0}^\infty \frac{1}{n!} h^{\otimes n}$ in $T_s((\mathcal{H})) \subset T_0((\mathcal{H}))$ (with $h^{\otimes 0} := 1$ and $0! := 1$). Then we have:
\begin{equation}
    \langle e(h_1), e(h_2) \rangle_{T_s((\mathcal{H}))} = \exp(\langle h_1, h_2 \rangle_{\mathcal{H}}).
\end{equation}
\end{lem}
\begin{proof}
This is a straightforward computation.
\end{proof}

Next we recall the definition of a kernel. Denote by $\mathbb{F}$ a field (e.g., $\RR$ and $\CC$). 

\begin{defn} (Kernel) Let $\chi$ be a nonempty set. Then a function $K: \chi \times \chi \to \mathbb{F}$ is called a kernel on $\chi$ if there exists a $\mathbb{F}$-Hilbert space $\mathcal{H}$ and a map $\phi: \chi \to \mathcal{H}$ such that for all $x, x' \in \chi$,  we have $K(x,x') = \langle \phi(x'), \phi(x) \rangle_{\mathcal{H}}$. 
We call $\phi$ a feature map and $\mathcal{H}$ a feature space of $K$.
\end{defn}
By definition, kernels are positive definite. \\

\begin{proof}(Proof of Proposition 3.2) 
Let $L \in \ZZ_+$. Viewing $\mathcal{H} := \mathcal{P} \otimes \RR^m$ as a set, it follows from  Theorem 4.16 in \citep{steinwart2008support} that  $K(\vecc{v},\vecc{w}) = \langle S(\vecc{v}), S(\vecc{w})  \rangle_{T^{\otimes L}_s((\mathcal{H}))}$ is a kernel over $\mathcal{H}$ with the feature space $T^{\otimes L}_s((\mathcal{H}))$ since it is an inner product on the $L$-fold symmetric Fock space $T^{\otimes L}_s((\mathcal{H}))$, which is a $\RR$-Hilbert space. The associated feature map is $\phi(\vecc{v}) = S(\vecc{v})$ for $\vecc{v} \in \mathcal{H}$. The last statement in the proposition then follows from Theorem 4.21 in \citep{steinwart2008support}. 
\end{proof}

\subsection{Proof of Theorem 3.5}

\begin{proof}(Proof of Theorem 3.5 -- dual formulation)\\ 
(a) First, note that by Lemma \ref{bv} any increasing function on $[0,T]$ has finite one-variation on $[0,T]$. Since  components of $\vecc{v}$, being real-valued polynomials on $[0,T]$, are linear combinations of monomials and linear combinations of functions of bounded variation on $[0,T]$ is also of bounded variation on $[0,T]$, the components of $\vecc{v}$ are also of bounded variation on $[0,T]$. Since product of functions of bounded variation on $[0,T]$ is also of bounded variation  on $[0,T]$, the $\vecc{X}_n = \vecc{v} \otimes \vecc{u}_n$ are of bounded variation on $[0,T]$, and thus the signature of the $\vecc{X}_n$ is well-defined.

By Proposition 3.2 in the main text, $\langle S(\vecc{X}_n), S(\vecc{X})  \rangle_{T_s((\mathcal{H}))} $ is a kernel on $\mathcal{H} := \mathcal{P} \otimes \RR^m$ with the RKHS $\mathcal{R}_1$. Therefore, the result in (a) follows from a straightforward application of Theorem 1 in \citep{scholkopf2001generalized}.  \\

\noindent (b) For $L \in \ZZ_+$ we compute:
\begin{align}
   \langle S(\tilde{\vecc{X}}_n), S(\tilde{\vecc{X}})  \rangle_{T^{\otimes L}_s((\mathcal{H}))} &=
\left\langle \bigotimes_{l=1}^L \exp(\Delta \vecc{X}_n^{(l)}), \bigotimes_{l=1}^L \exp(\Delta \vecc{X}^{(l)})  \right\rangle_{T_s^{\otimes L}((\mathcal{H}))}    \\
   &= \prod_{l=1}^L \left\langle e(\Delta \vecc{X}_n^{(l)}), e(\Delta \vecc{X}^{(l)}) \right\rangle_{T_s((\mathcal{H}))} \\
   &= \prod_{l=1}^L \exp(\langle \Delta \vecc{X}_n^{(l)}, \Delta \vecc{X}^{(l)} \rangle_{\mathcal{H}}),
\end{align}
where we have used Lemma \ref{piece} in the first line and Lemma \ref{exp_prop} in the last line above. The proof is done. 
\end{proof}

\section{An Approximation Result for SRNNs}
\label{approx}

In this section we justify why Assumption \ref{imp_ass}, in particular the analyticity of the coefficients defining the SRNNs, is in some sense not too restrictive. We need the following extension of Carleman's theorem  \citep{gaier1987lectures} to multivariate functions taking values in multi-dimensional space. Note that Carleman's theorem itself can be viewed as an extension of Stone-Weierstrass theorem  to non-compact intervals.


\begin{thm} \label{thm_c1} If $\epsilon(\vecc{x}) > 0$ and $\vecc{f}: \RR^n \to \RR^m$ are arbitrary continuous functions, then there is an entire function $\vecc{g}: \CC^n \to \CC^m$  such that for all real $\vecc{x} \in \RR^n$ (or, real part of $\vecc{z} \in \CC^n$), 
$$|\vecc{f}(\vecc{x}) - \vecc{g}(\vecc{x})| < \epsilon(\vecc{x}).$$
\end{thm}
\begin{proof}
This is a straightforward extension of the main result  in \citep{scheinberg1976uniform} to $\RR^m$-valued functions. \end{proof}

The above theorem says that any $\RR^m$-valued continuous function on $\RR^n$ can be approximated arbitrarily closely by the restriction of an entire function of $n$ complex variables. In the univariate one-dimensional case,  an entire function $g$ has an everywhere convergent Maclaurin series $g(z)=\sum_n a_n z^n$. The restriction of $g$ to $\RR$, $\tilde{g}=g|_{\RR}$ therefore also admits an everywhere convergent series $\tilde{g}(x)=\sum_n a_n x^n$, thus analytic. Similar statements in the multi-dimensional and multivariate case hold.


Let $(\vecc{h}'_t,\vecc{y}'_t)$ be the hidden state and output of the SRNN defined by 
\begin{align}
    d\vecc{h}_t' &= -\vecc{\Gamma} \vecc{h}_t' dt + \vecc{a}'(\vecc{W} \vecc{h}_t'+\vecc{b}) dt + \vecc{C} \vecc{u}_t dt + \vecc{\sigma} d\vecc{W}_t, \label{reg1} \\
    \vecc{y}_t' &= \vecc{f}'(\vecc{h}_t'), \label{reg2}
\end{align}
and satisfying  Assumption \ref{imp_ass} with $\vecc{h}$, $\overline{\vecc{h}}$ replaced by $\vecc{h}'$, $\overline{\vecc{h}}'$ respectively, the  functions $\vecc{a}$, $\vecc{f}$ replaced by $\vecc{a}'$, $\vecc{f}'$ respectively, and with (c) there replaced by: \\

(c') The coefficients $\vecc{a}': \RR^n \to \RR$ and $\vecc{f}': \RR^n \to \RR^p$ are activation functions, i.e., bounded, non-constant, Lipschitz continuous functions. \\

We call this SRNN a {\it regular SRNN}.\\ 

On the other hand,  let $(\vecc{h}_t, \vecc{y}_t)$ be the hidden state and output of  the SRNN defined in (1)-(3) in the main paper and satisfying Assumption \ref{imp_ass}, with $\vecc{h}_0 = \vecc{h}_0'$. We call such SRNN an {\it analytic SRNN}. 

\begin{thm} (Approximation of a regular SRNN by analytic SRNNs)\\
Let $p > 0$.  Given a regular SRNN, there exists an analytic SRNN  such that for  $\epsilon >0$ arbitrarily small, $\sup_{t \in [0,T]} \mathbb{E}  |\vecc{y}_t - \vecc{y}'_t|^p < \epsilon$. 
\end{thm}

\begin{proof}
Let $\epsilon > 0$  be given and the regular SRNN be defined by \eqref{reg1}-\eqref{reg2}. 

By Theorem \ref{thm_c1}, for any $\eta_i(\vecc{h}) > 0$ ($i=1,2$), there exists analytic functions $\vecc{a}$ and $\vecc{f}$ (restricted to $\RR^n$) such that for all $\vecc{h} \in \RR^n$, 
\begin{align}
|\vecc{a}(\vecc{h})-\vecc{a}'(\vecc{h})| < \eta_1(\vecc{h}), \\ 
|\vecc{f}(\vecc{h}) - \vecc{f}'(\vecc{h})| < \eta_2(\vecc{h}).
\end{align}
Therefore, by Lipschitz continuity of $\vecc{a}'$ and $\vecc{f}'$, we have:
\begin{align}
|\vecc{a}(\vecc{h}) - \vecc{a}'(\vecc{h}')| \leq |\vecc{a}(\vecc{h}) - \vecc{a}'(\vecc{h})| + |\vecc{a}'(\vecc{h}) - \vecc{a}'(\vecc{h}')| <\eta_1(\vecc{h}) +  L_1| \vecc{h}-\vecc{h}'|, \label{120}\\
|\vecc{f}(\vecc{h}) - \vecc{f}'(\vecc{h}')| \leq |\vecc{f}(\vecc{h}) - \vecc{f}'(\vecc{h})| + |\vecc{f}'(\vecc{h}) - \vecc{f}'(\vecc{h}')| <\eta_2(\vecc{h}) +  L_2| \vecc{h}-\vecc{h}'| \label{164}, 
\end{align} 
where $L_1, L_2 > 0$ are the Lipschitz constants of $\vecc{a}'$ and $\vecc{f}'$ respectively.

From \eqref{e1}-\eqref{e3} and \eqref{reg1}-\eqref{reg2},  we have, almost surely,
\begin{align}
\vecc{h}_t &= \vecc{h}_0 -  \vecc{\Gamma} \int_0^t \vecc{h}_s ds +  \int_0^t \vecc{a}(\vecc{W} \vecc{h}_s + \vecc{b}) ds + \int_0^t \vecc{C} \vecc{u}_s ds + \vecc{\sigma} \vecc{W}_t, \\
\vecc{h}'_t &= \vecc{h}_0 -  \vecc{\Gamma} \int_0^t \vecc{h}'_s ds +  \int_0^t \vecc{a}'(\vecc{W} \vecc{h}'_s + \vecc{b}) ds + \int_0^t \vecc{C} \vecc{u}_s ds + \vecc{\sigma} \vecc{W}_t.
\end{align}

Let $p\geq 2$ first in the following. For $t \in [0,T]$,
\begin{align}
| \vecc{h}_t - \vecc{h}_t' |^p 
&\leq 2^{p-1} \left(  \int_0^t |\vecc{\Gamma}(\vecc{h}_s - \vecc{h}_s')|^p ds  +  \int_0^t |\vecc{a}(\vecc{h}_s)  - \vecc{a}'(\vecc{h}_s')|^p ds  \right).
\end{align}
Using boundedness of $\vecc{\Gamma}$ and \eqref{120}, we estimate
\begin{align}
\mathbb{E}  | \vecc{h}_t - \vecc{h}_t' |^p
&\leq C_1(p,T) \eta(T)  + C_2(L_1,p) \int_0^t  \mathbb{E} |\vecc{h}_s - \vecc{h}_s' |^p ds, 
\end{align}
where $\eta(T) = \mathbb{E} \sup_{s \in [0,T]} ( |\eta_1(\vecc{h}_s)|^p) $,  $C_1(p,T) > 0$ is a constant depending on $p$ and $T$, and $C_2(L_1,p) > 0$ is a constant depending on $L_1$ and $p$. 
Therefore, 
\begin{align}
\sup_{t \in [0,T]} \mathbb{E}  | \vecc{h}_t - \vecc{h}_t' |^p
&\leq C_1(p,T) \eta(T)  + C_2(L_1,p) \int_0^T  \sup_{u \in [0,s]} \mathbb{E} |\vecc{h}_u - \vecc{h}'_u |^p ds. 
\end{align}

Applying Gronwall's lemma,
\begin{align}
\sup_{t \in [0,T]} \mathbb{E}  | \vecc{h}_t - \vecc{h}_t' |^p &\leq C_1(p,T) \eta(T) e^{C_2(L_1,p)T}. \label{gron}
\end{align}
Thus, using \eqref{164} and \eqref{gron} one can estimate:
\begin{align}
\sup_{t \in [0,T]} \mathbb{E} | \vecc{y}_t - \vecc{y}_t'|^p &\leq 2^{p-1} L_2^p \sup_{t \in [0,T]} \mathbb{E} |\vecc{h}_t - \vecc{h}_t'|^p + 2^{p-1} \mathbb{E} \sup_{s \in [0,T]} |\eta_2(\vecc{h}_s)|^p \\
&\leq C_3(L_2,p,T) \eta(T) e^{C_2(L_1,p)T} + 2^{p-1} \mathbb{E} \sup_{s \in [0,T]} |\eta_2(\vecc{h}_s)|^p \\
&\leq C_4(L_1,L_2,p,T) \beta(T),
\end{align}
where  $\beta(T) =  \mathbb{E} \sup_{s \in [0,T]} \left( \sum_{i=1}^2  |\eta_i(\vecc{h}_s)|^p\right)$ and the $C_i > 0$ are constants depending only on their arguments.

Since the $\eta_i > 0$ are arbitrary, we can choose them so that $\beta(T) < \epsilon/C_4$. This concludes the proof for the case of $p\geq 2$. 

The case of $p \in (0,2)$ then follows by an application of H\"older's inequality: 
taking $q>2$ so that $p/q<1$,
\begin{align}
\sup_{t \in [0,T]} \mathbb{E} | \vecc{y}_t - \vecc{y}_t' |^p &\leq \sup_{t \in [0,T]} (\mathbb{E} (| \vecc{y}_t - \vecc{y}_t' |^p)^{q/p})^{p/q}  = \left( \sup_{t \in [0,T]} \mathbb{E} (| \vecc{y}_t - \vecc{y}_t' |^q) \right)^{p/q} \\ 
&\leq C_5(L_1,L_2,p,q,T) \beta^{p/q}(T),
\end{align}
for some constant $C_5 > 0$, and choose $\beta(T)$ such that $\beta(T) < \left(\frac{\epsilon}{C_5}\right)^{q/p}$.


\end{proof}

\newpage

\end{document}